%% file: MatchFame.tex
\documentclass[10pt,twocolumn,letterpaper]{article}
\input{macro_file}

\usepackage[pagenumbers]{cvpr} 

\usepackage[subtle]{savetrees}[lists=tight,tracking=tight,charwidths=tight,floats=tight,leading=tight,paragraphs=tight,wordspacing=tight]
\usepackage{graphicx}
\usepackage{amsmath}
\usepackage{amssymb}
\usepackage{booktabs}
\usepackage{algorithm}
\usepackage{algorithmic}
\newtheorem{theorem}{Theorem}
\newtheorem{definition}{Definition}
\newtheorem{lemma}{Lemma}

\usepackage{multirow}

\usepackage[pagebackref,breaklinks,colorlinks]{hyperref}

\usepackage[capitalize]{cleveref}
\crefname{section}{Sec.}{Secs.}
\Crefname{section}{Section}{Sections}
\Crefname{table}{Table}{Tables}
\crefname{table}{Tab.}{Tabs.}

\usepackage{mathtools}

\def\cvprPaperID{*****} 
\def\confName{CVPR}
\def\confYear{2022}

\begin{document}

\title{Fast, Accurate and Memory-Efficient Partial Permutation Synchronization}


\makeatletter
\newcommand{\printfnsymbol}[1]{%
  \textsuperscript{\@fnsymbol{#1}}%
}

\author{
   Shaohan Li\printfnsymbol{1}\hspace{2cm}
   Yunpeng Shi\printfnsymbol{2}\hspace{2cm}

   Gilad Lerman\printfnsymbol{1}\\
  \printfnsymbol{1}School of Mathematics, University of Minnesota\\
  \printfnsymbol{2}Program in Applied and Computational Mathematics, Princeton University\\
 \tt{\{li000743, lerman\}@umn.edu}, yunpengs@princeton.edu
  }

\maketitle

\begin{abstract}
Previous partial permutation synchronization (PPS) algorithms, which are commonly used for multi-object matching, often involve computation-intensive and memory-demanding matrix operations. These operations become intractable for large scale structure-from-motion datasets. For pure permutation synchronization, the recent Cycle-Edge Message Passing (CEMP) framework suggests a memory-efficient and fast solution. Here we overcome the restriction of CEMP to compact groups and propose an improved algorithm, CEMP-Partial, for estimating the corruption levels of the observed partial permutations. It allows us to subsequently implement a nonconvex weighted projected power method without the need of spectral initialization. The resulting new PPS algorithm, MatchFAME
(Fast, Accurate and Memory-Efficient Matching), only involves sparse matrix operations, and thus enjoys lower time and space complexities in comparison to previous PPS algorithms. We prove that under adversarial corruption, though without additive noise and with certain assumptions, CEMP-Partial is able to exactly classify corrupted and clean partial permutations. We demonstrate the state-of-the-art accuracy, speed and memory efficiency of our method on both synthetic and real datasets.
\end{abstract}

\section{Introduction}
\label{sec:intro}
The problem of partial permutation synchronization (PPS) naturally arises from the task of multi-object matching (MOM). MOM assumes multiple objects (e.g. images), where each single object contains some keypoints associated with underlying distinct labels. The set of all distinct labels is called the universe.
Ideally, any two keypoints, from different objects, that share a common label should be matched.
Given partially observed and corrupted pairwise keypoint matches, MOM asks to recover the ground truth labels of each keypoint, or equivalently, the keypoint-to-universe matches. In structure from motion (SfM), where the objects are images, MOM is often referred to as multi-image matching. Here a keypoint is characterized by a specific location in the image and its associated label is the index of its corresponding 3D point, which can be viewed at this location of the image. In this case, the initial pairwise keypoint matches are typically obtained by SIFT \cite{sift04} and the MOM problem  asks to identify the corresponding 3D point index for each keypoint in each image, up to an arbitrary permutation of the indices.


The mathematical formulation of PPS represents the images
in the latter problem as nodes of an unweighted and undirected graph, which is commonly referred to as the viewing graph, and further represents both
``relative" keypoint-to-keypoint  matches
among pairs of images and the ``absolute" keypoint-to-universe matches
as partial permutation matrices. We recall that a partial permutation matrix is binary with at most one nonzero element at each row and column.
PPS thus asks to recover the ``absolute'' partial permutations (which are associated with nodes of the graph) given possibly corrupted and noisy measurements of the ``relative'' partial permutations (which are associated with edges of the graph).
We remark that when restricting the partial permutation matrices to be full permutations (bi-stochastic, binary and square), PPS reduces to permutation synchronization (PS), which is a special case of the group synchronization problem. Probably, the most well-known group synchronization problem is rotation averaging \cite{HartleyAT11_rotation, ChatterjeeG13_rotation, MPLS} where the group is $SO(3)$. Various approaches from rotation averaging, such as the spectral method \cite{singer2011angular} and SDP relaxation \cite{wang_singer_2013}, can be similarly employed to PS \cite{deepti, Huang13}. However, when considering PS for image matching, all images must share the same set of keypoints. This is restrictive, since images taken from different viewing directions may share very few (or even no) keypoints. Therefore, PPS is more realistic for image matching, and in particular, SfM, than PS.

The PPS problem is challenging for three different reasons.
First of all, the corruption of pairwise measurements in real data can be highly nonuniform, which violates the common assumptions of uniform corruption in \cite{deepti, chen_partial}. Indeed, as is pointed out in \cite{IRGCL}, the corruption in keypoint matches in real data can
concentrate at local regions of the viewing graph.
Second, the number of rows or columns of each partial permutation can be in the order of hundreds or even higher, which is much larger than dimension 3 in rotation averaging. This makes PPS a computationally-intensive task in comparison to rotation averaging. Finally, the relative partial permutations in PPS are no longer square matrices, as in PS, and can have very different sizes and sparsity levels. Consequently, they may introduce additional bias and numerical instability to common PPS algorithms.

This work addresses the above challenges and develops a fast, accurate and memory efficient PPS algorithm that works well for nontrivial corruption models and large scale real data.

\subsection{Related Works}
The first PS algorithm \cite{deepti}, which is commonly referred to as Spectral, computes  the top $m$ eigenvectors of the block matrix of relative permutations, where $m$ is the universe size. It then obtains the absolute partial permutations by projecting the blocks of the eigenmatrix to full permutations using the Hungarian algorithm \cite{Munkres}.  It can be easily adapted to PPS tasks by using other heuristic projection methods \cite{ConsistentFeature, inv_semi_group}. A similar PPS algorithm is MatchEig \cite{MatchEig}, which applies a faster heuristic projection to partial permutations, and an additional hard thresholding step. Although it achieves slight speedup in comparison to Spectral, the hard thresholding can result in overly sparse keypoint matches on some datasets. A theoretically guaranteed SDP relaxation method for PPS, MatchLift \cite{chen_partial}, was proposed for near-optimal handling of the uniform corruption model. However, the SDP relaxation suffers from high computational complexity and is often several orders of magnitude slower than spectral-based methods. MatchALS \cite{MatchALS} replaces the SDP constraint of MatchLift by linear ones, which yield significant speedup. However, it is still much slower than spectral-based methods and is not scalable to even medium-size datasets.  Moreover, the accuracy of \cite{chen_partial, MatchALS} are not competitive on some real datasets as reported in \cite{MatchEig}.  For PS, \cite{Birdal_2019_CVPR} relaxes the space of permutations to the Birkhoff polytope and solves the maximum a-posteriori (MAP) problem on this relaxed manifold. This method relies on a special probabilistic model for permutations, but it has no convergence guarantees and is restricted to PS.

Most importantly, all the aforementioned PS/PPS methods are memory demanding and thus cannot handle large scale SfM datasets such as Photo Tourism \cite{photo_tourism}. For spectral and MatchEig, the top $m$ eigenvectors form a dense $M \times m$ matrix. For Photo Tourism, $m > 10^4$ and $M \approx 10^6$ and thus Spectral and MatchEig require at least 80 GB memory and cannot be implemented on a  personal computer. MatchALS and MatchLift further require eigenvalue computation of an $M \times M$ dense matrix.
These dense matrix operations also increase the time complexity.

A faster and more memory-efficient algorithm (excluding its initialization stage) is the projected power method (PPM). PPM is a nonconvex method based on blockwise power iterations followed by a projection onto the permutation matrices \cite{Chen_PPM}. It can be equivalently viewed as a special case of the projected block coordinate descent algorithm assuming the least squares objective function. It was applied for PS, but as we show it can be easily extended to PPS. We note that in PPS each power iteration only stores $M \times m$ sparse binary matrices for estimating the absolute partial permutations. The number of nonzero elements in these sparse matrices is at most $m$, thus PPM is at least 10,000 times more memory-efficient than spectral and MatchEig on large SfM data. Moreover, since the above power iterations only operate on sparse matrices, its time complexity is also significantly smaller than those of spectral-based methods. However, as far as we know, PPM was only tested in \cite{Chen_PPM, PPM_vahan, IRGCL} on both $\mathbb Z_m$-synchronization and PS, and was never applied to the PPS problem. We remark that there are a couple of limitations that prevent the application of PPM to large-scale real SfM datasets. First of all, as a nonconvex method, it is very sensitive to the initialization of the absolute permutations. A common initializer for PPM is Spectral, which makes PPM memory demanding, regardless how memory-efficient the power iteration is, and also slower. Second, same as Spectral and SDP methods, PPM minimizes the least squares energy, making it nonrobust under nonuniform and adversarial corruption as shown in \cite{IRGCL}.

The recent theoretically-guaranteed cycle-edge message passing (CEMP) algorithm \cite{cemp} opens the door for fast, memory efficient, and outlier-robust implementation for compact group synchronization without spectral initialization. Different from the previous cycle-consistency-based methods \cite{shen2016, Zach2010, non-seq, viewing_graph}, it uses a fast iterative message passing scheme to globally estimate the corruption levels of the given pairwise measurements. It is numerically demonstrated in \cite{cemp} that CEMP is memory-efficient and fast for $SO(d)$-synchronization, especially for large $d$. For permutation synchronization, \cite{IRGCL} proposed an efficient implementation of CEMP. In particular, it iteratively reweighted the edges of the viewing graph using CEMP-estimated corruption levels, and simultaneously applied a weighted spectral/PPM method.
However, these frameworks were only fully developed for compact group synchronization. The extension of CEMP to partial permutations is nontrivial, since partial permutations do not form a group and the theory of CEMP to date no longer holds in this new regime.

FCC \cite{FCC}, which was published after the submission of this work, assigns for each keypoint match a confidence score for being a correct match. It is fast, accurate and memory efficient, and seems to comparably perform to the proposed method. Its different graph model might be more tolerant to additive noise. However, our method has several advantages. First, it enjoys some theoretical guarantees. Second, its refined matches are automatically cycle-consistent. Last, its space complexity is lower than that of FCC with default parameters.

\subsection{Contributions of This Work}
The main contributions of this work are as follows:
\begin{itemize}[noitemsep,nolistsep]
    \item We overcome the restriction of CEMP to compact groups, and extend it to PPS. The new CEMP-Partial algorithm can be applied to general SfM matching data, and is theoretically guaranteed under adversarial corruption.
    \item We propose MatchFAME (Fast, Accurate and Memory-Efficient Matching) for PPS. It combines CEMP-Partial with weighted PPM.  It only involves sparse matrix operations and thus enjoys significantly lower time and space complexities than previous PPS methods.
    \item We demonstrate the accuracy and efficiency of MatchFAME on synthetic and real datasets in comparison to the current state-of-the-art PPS methods.
\end{itemize}

\section{Partial Permutation Synchronization}
\label{sec:pps}

Recall that a partial permutation matrix is a binary matrix that has at most one nonzero element at each row and column. We note that it is different from a full permutation whose rows and columns have exactly one nonzero element. A partial permutation can be rectangular, whereas the full one has to be a square.
We denote by $\mathcal{P}^{l_1,l_2}$ the space of $l_1 \times l_2$ partial permutations, and by $\mathcal{P}^{m}_F$ the space of $m \times m$ full permutations. We also denote $[n]=\{1,2,..., n\}$ for $n \in \mathbb{N}$.
Using the above notation, we formally state the PPS problem as follows. The problem assumes a graph $G=([n],E)$ with $n$ nodes and underlying unknown ground-truth partial permutations $\{\bP_i^{*}\}_{i\in [n]}$ of size $m_i \times m$ associated with the nodes, where $m$ is fixed and $m\geq m_i$ for all $i$. 
It further assumes that for each edge $ij \in E$, a pairwise partial permutation $\bX_{ij}$ is observed, which is viewed as a measurement of the ground-truth partial permutation, $\bX_{ij}^*:=\bP_i^* \bP_j^{*T}$.  The PPS problem asks to recover $\{\bP_i^{*}\}_{i\in [n]}$ from $\{\bX_{ij}\}_{ij\in E}$.
In practice, PPS algorithms often operate on the block matrix $\bX = (\bX_{ij})_{1 \le i,j \le n}$, where $\bX_{ij} = 0$  for $ij \not \in E$.

The adversarial corruption model 
partitions the edge set $E$ into a set of clean (good) edges, $E_g$, and a set of corrupted (bad) edges, $E_b$, where for $ij \in E_g$, $\bX_{ij}=\bX_{ij}^*$, and for $ij \in E_b$, $\bX_{ij}\neq \bX_{ij}^*$.
It is adversarial since it does not make assumptions on the distribution of the corrupted partial permutations and the graph topology; though we may add some assumptions.

In multi-image matching, $G([n],E)$ is referred to as the viewing graph, $m$ is the number of 3D points and $n$ is the number of images. Each graph node $i$ is associated with an image with $m_i$ keypoints, and $\bP_i^*$ encodes its ground truth keypoint-universe matches. Specifically, $\bP_i^*(k,l)=1$ if and only if the $k$-th keypoint in image $i$ corresponds to the $l$-th point in the 3D point cloud. For each edge $ij \in E$, the partial permutation $\bX_{ij}$ represents the observed keypoint matches between images $i$ and $j$ (obtained by e.g., SIFT). We note that $\bX_{ij}(k,l)=1$ if and only if we observe a match between the $k$-th keypoint in image $i$ and the $l$-th keypoint in the $j$-th image. We denote $M: = \sum_{i=1}^n m_i$ and note that the block matrix $\bX$ is of size $M \times M$.
At last we comment that in multi-image matching one mainly cares about improving the keypoint matches. Therefore instead of the estimates $\{\hat\bP_{i}\}_{i \in [n]}$ of absolute permutations, it is common to output the estimates of relative permutations, $\bZ_{ij} = \hat\bP_{i}\hat\bP_{j}^\top$ for any $ij\in E$.





\section{Proposed Method}
\label{sec:proposed}
\subsection{Brief Review of CEMP}
\label{sec:cemp_review}
We 
focus on the case of PS with the distance
\begin{equation}
\label{eq:distance_cemp}
d(\bX_1, \bX_2) = 1-\langle\bX_1\,, \bX_2\rangle/m,
\ \text{ for } \bX_1, \ \bX_2 \in \mathcal{P}_F^m.
\end{equation}
For $ij \in E$, we define the ground-truth corruption level by
$$s_{ij}^* = d(\bX_{ij}, \bX_{ij}^*).$$
CEMP uses cycle-consistency information to estimate these corruption levels.
Recall that a 3-cycle, $ijk$, is a path in $G([n],E)$ containing the nodes $i$, $j$, $k \in [n]$. In our previous and current work we focus on 3-cycles for simplicity and efficient computation and refer to them just as cycles. One can extend our methods to higher-order cycles.
In PS, a cycle $ijk \subseteq E$ is consistent if and only if $\bX_{ij}\bX_{jk} = \bX_{ik}$ (equivalently, $\bX_{ij}\bX_{jk}\bX_{ki}=\bI$, where $\bI$ denotes the identity matrix).
We define the cycle inconsistency of the cycle $ijk$ as
$$d^{\text{CEMP}}_{ijk}:=d(\bX_{ij}, \bX_{ik}\bX_{kj}) = 1-\langle\bX_{ij}\,, \bX_{ik}\bX_{kj}\rangle/m.$$
CEMP iteratively approximates the corruption levels as follows
\begin{align}
\label{eq:sij_cemp_est}
    s_{ij}^*
    \approx \sum_{k} w_{ijk} d^{\text{CEMP}}_{ijk},
\end{align}
where the weights are updated at each iteration  using  improved estimates of the corruption levels (we omit their formulas).

The bi-invariance property of the distance in \eqref{eq:distance_cemp} implies \cite{cemp}
\begin{equation}
\label{eq:good_cycle_cemp}
    d_{ijk}^{\text{CEMP}} = s_{ij}^* \text{ whenever } ik, \ jk\in E_g.
\end{equation}
CEMP chooses specific weights in \eqref{eq:sij_cemp_est} that aim to highlight the case in \eqref{eq:good_cycle_cemp}. That is, for $ij \in E$, they are ideally close to 1 when $ik \in E_g$ and $jk \in E_g$, and close to 0 otherwise.
Under some conditions, CEMP exactly recovers the corruption levels \cite{cemp}.




\subsection{A Cycle Inconsistency Measure for PPS}
\label{sec:cycle_consistency_def}
For PPS, we say that a cycle $ijk$ is consistent whenever $\bX_{ij}\bX_{jk}\leq \bX_{ik}$, $\bX_{jk}\bX_{ki}\leq \bX_{ji}$ and $\bX_{ki}\bX_{ij}\leq \bX_{kj}$ (the notation $\bA \leq \bB$ corresponds to the component-wise inequality, i.e., $A_{ij} \leq B_{ij}$ for all indices $i$, $j$).
We further motivate this definition in the supplemental material.

In order to define an inconsistency measure for PPS, we make several observations. We first formally generalize \eqref{eq:distance_cemp} and define the following conditional dissimilarity function, or divergence,
where $\bX_1$, $\bX_2 \in \mathcal{P}^{l_1,l_2}$:
$$
    d(\bX_1 | \bX_2):= 1-
    {\langle \bX_1, \bX_2\rangle} / {\nnz(\bX_2)},
$$
where $\nnz(\bX_2)$ is the number of nonzero elements in the matrix $\bX_2$.
We note that 
    $d(\bX_1 | \bX_2)= 0$
    if and only if
$\bX_2 \leq \bX_1$
and therefore $\bX_{ik}\bX_{kj}\leq \bX_{ij}$ if and only if
    $$\tr(\bX_{ji}\bX_{ik} \bX_{kj})
    \equiv \langle \bX_{ij}, \bX_{ik} \bX_{kj}
    \rangle
= \nnz(\bX_{ik} \bX_{kj}).$$
Thus, for a cycle $ijk$,
we define $n_i = \nnz(\bX_{ki}\bX_{ij})$, $n_j = \nnz(\bX_{kj}\bX_{ji})$, $n_k = \nnz(\bX_{ik}\bX_{kj})$,  $n_{\Delta} = \tr(\bX_{ij}\bX_{jk}\bX_{ki})$, and conclude  
that $ijk$ is cycle-consistent if and only if
$3n_{\Delta} = n_i+n_j+n_k$.
In view of this observation, we suggest the following cycle inconsistency measure for PPS:
\begin{align}\label{eq:dijk}
    d_{ijk} = 1- 3n_{\Delta}/(n_i+n_j+n_k).
\end{align}
We further interpret this measure in the supplemental material.





\subsection{The CEMP-Partial Algorithm}
\label{sec:cemp_partial}
The CEMP algorithm for PS is motivated by \eqref{eq:sij_cemp_est} and \eqref{eq:good_cycle_cemp}.
By replacing $d_{ijk}^{\text{CEMP}}$ in the the CEMP algorithm  by our new $d_{ijk}$ in \eqref{eq:dijk}, one can obtain the CEMP-Partial algorithm for PPS, which is sketched in Algorithm \ref{alg:cemp-partial}. This algorithm iteratively updates the corruption levels according to \eqref{eq:alg-sij} below, where the weights $w_{ijk}^{(t)}$ will be clarified below and $N_{ij}:=\{k\in [n]: ik, jk \in E\}$, $ij\in E$, is the set of nodes that form cycles with the edge $ij$. Note that \eqref{eq:alg-sij} is analogous to \eqref{eq:sij_cemp_est} in the case of PS.
While in PS the aim of such a formula is to eventually yield a good estimate for the corruption levels, in PPS we only aim to cluster uncorrupted and corrupted edges with low (close to zero) and sufficiently high corruption levels, respectively; we show that this is indeed possible in \S\ref{sec:theory}.
Before the iterations, the corruption levels are initialized in \eqref{eq:alg-init} as a similar average but with uniform weights. Given the corruption levels of the previous iteration, the weights of the current iteration are computed by \eqref{eq:alg-wijk}, where $\beta_t$ is a fixed parameter at iteration $t$.
This formula aims to assure that $w_{ijk}^{(t)}$ is large (close to 1) whenever both $ik$ and $jk$ are good and close to zero otherwise, so that the estimated $s_{ij}^{(t+1)}$ in \eqref{eq:alg-wijk} is approximately an average of only those $d_{ijk}$'s with good edges $ik$ and $jk$.
Such a behavior of $w_{ijk}^{(t)}$ occurs when the estimates of the corruption levels  $s_{ij}$ is close to zero when $ij$ is a good edge and sufficiently far from zero otherwise. In fact, alternatively updating both the weights and the corruption levels aims to result in such a property (see \S\ref{sec:theory}).


\begin{algorithm}[h]
\caption{CEMP-Partial}
\label{alg:cemp-partial}
\begin{algorithmic}
\REQUIRE $\{d_{ijk}\}_{ij\in E, k\in N_{ij}}$ (see  \eqref{eq:dijk}), number of iterations $T$, increasing  $\{\beta_t\}_{t=1}^T$\\
\STATE \textbf{Steps:}
\vspace*{-\baselineskip}
\STATE
\begin{equation}\label{eq:alg-init}
s_{ij}^{(0)} = \frac{1}{|N_{ij}|} \sum_{k\in N_{ij}} d_{ijk} \  \text{ for }   \ ij\in E
\end{equation}
\vspace*{-\baselineskip}
\FOR {$t=0:T-1$}
\vspace*{-\baselineskip}
\STATE \begin{align}
  &w_{ijk}^{(t)} =  \exp\mleft(-\beta_{t}\mleft(s_{ik}^{(t)}+s_{jk}^{(t)}\mright)\mright),\  k\in N_{ij},\ ij\in E \label{eq:alg-wijk}\\
    &s_{ij}^{(t+1)}
    =\frac{1}{Z_{ij}^{(t)}}\sum\limits_{k\in N_{ij}}w_{ijk}^{(t)} d_{ijk} \text{ for } \ ij\in E \label{eq:alg-sij}
\end{align}
\STATE
where $Z_{ij}^{(t)}=\sum\limits_{k\in N_{ij}}w_{ijk}^{(t)}$ is a normalization factor
\ENDFOR
\ENSURE $\hat{s}_{ij}:=s_{ij}^{(T)}$ \  for \ $ij\in E$
\end{algorithmic}
\end{algorithm}


\subsection{MatchFAME}
\label{sec:matchfame}
We propose MatchFAME that aims to address the main challenges of PPS (see \S\ref{sec:intro}). MatchFAME combines  CEMP-Partial with a weighted PPM method.
The original (unweighted) PPM is an iterative procedure that aims to
minimize the least squares energy
   $\sum_{ij\in E}\|\bP_i\bP_j^\top - \bX_{ij}\|_F^2$,
under the constraint $\bP_i \in \mathcal{P}_F^m$ for $i\in [n]$. Given the estimated absolute permutations for the different nodes at the $t$-th iteration, $\{\bP_i^{(t)}\}_{i \in [n]}$, PPM estimates each permutation on node $i$ in the next iteration as
\begin{align}\label{eq:PPM}
    \bP_i^{(t+1)} = \textbf{Proj} 
    \mleft(\sum_{j\in N_i}\bX_{ij}\bP_j^{(t)}/|N_i|\mright),
\end{align}
where $N_i$ denotes the neighboring nodes of $i$, and \textbf{Proj} is the projection onto the space of permutations, which can be computed by the Hungarian algorithm \cite{Munkres}. Intuitively, at iteration $t$ of PPM, each $j\in N_i$ proposes the ``local" estimate of $\bP_i^*$: $\bX_{ij}\bP_j^{(t)}$, and $\bP_i$ in the new iteration is updated by the average of these local estimates followed by a projection. However, these local estimates are only accurate when $\bX_{ij}\approx \bX_{ij}^*$ and $\bP_j^{(t)}\approx \bP_j^*$ for all $j\in N_i$. This makes PPM sensitive to both initialization and edge corruption, and thus a naive generalization of PPM to partial permutations is not sufficient to handle the PPS challenges described in \S\ref{sec:intro}.

To address the first PPS challenge of nonuniform corruption (see \S\ref{sec:intro}), we assign a weight $w_{ij}$ to each $ij\in E$, where $w_{ij}$ depends on $\hat{s}_{ij}$, the estimated corruption level by CEMP-Partial, and a parameter $\gamma$, which measures the confidence of the estimated corruption levels, as follows:
$w_{ij}=\exp(-\gamma \hat{s_{ij}})$. We can thus implement a weighted PPM iteration for each $i\in [n]$:
\begin{align}\label{eq:WPPM}
    \bP_i^{(t+1)} = \textbf{Proj} \left(\sum_{j\in N_i}\tilde w_{ij}\bX_{ij}\bP_j^{(t)}\right),
\end{align}
where $\bP_i^{(t)}, \bP_i^{(t+1)}\in \mathcal{P}^{m_i, m}$, $\tilde w_{ij}=w_{ij}/\sum_{j\in N_i} w_{ij}$ are the normalized weights and \textbf{Proj} is the heuristic and fast projection onto partial permutations described in \cite{MatchEig}. In such a way, the projected power iterations will focus on the clean edges and thus largely mitigate the sensitivity of standard PPM towards nonuniform topology of the corrupted subgraph, and nonuniform distribution of the corrupted partial permutations \cite{IRGCL}. We remark that given the ideal weights $w_{ij} = \mathbf{1}_{\{ij \in E_g\}}$, where $\mathbf{1}$ is the indicator function, the ground truth permutations, $\{\bP_i^*\}_{i\in [n]}$, form a fixed point of our weighted PPM.

The second PPS challenge of  highly-demanding computation (see \S\ref{sec:intro}) arises in PPM if it is initialized by Spectral or another standard PPS algorithm. In order to resolve this issue, we follow \cite{MPLS} and initialize our solution using a minimum spanning tree (MST), which depends on the output of CEMP-Partial. Specifically, we build a weighted graph where edge weights are the estimated corruption levels by CEMP-Partial. An MST is then extracted from the weighted graph. Note that it has the lowest average corruption levels among all other spanning trees.
We then use it to initialize the absolute permutations. We first arbitrarily assign to the root node the partial permutation  $\bI_{m_i\times m}$, which is an $m_i \times m$ matrix whose diagonal elements are 1 and the rest are 0.
We subsequently multiply relative permutations along the MST, namely applying
$\bP_j^{(t)} = \textbf{Proj}(\bX_{ji}\bP_i^{(t)})$
from the root to the leaves. We note that computation of the MST only uses the small $n \times n$ adjacency matrix of the graph, unlike  spectral initialization that involves the eigenvalue decomposition of a huge $M \times M$ matrix. As a result, our initialization is much faster and memory efficient, which breaks the computational bottleneck of standard PPM.

The last challenge of PPS is the uneven dimension and sparsity level of partial permutations. The main hurdle was the generalization of CEMP to PPS (see  \S\ref{sec:cemp_partial}). This way, the weights for our weighted PPM can be reliably estimated. Another issue that arises due to this challenge is that the MST initialization may be too sparse at the end of the spanning tree if some partial relative permutations in $\{\bX_{ij}\}_{ij \in E}$ are extremely sparse.  In some extreme cases, some columns of the estimated absolute permutations are zero. Therefore, after initializing the absolute permutations, we check if there is a zero column for the block column matrix of the estimated absolute permutations. If it is the case, we randomly fill one of the elements with 1.
In some special cases, where denser matrices are needed, one may fill instead zero rows (see supplemental material).

The full description of MatchFAME is in Algorithm \ref{alg:matchfame}. As is common in SfM, MatchFAME outputs refined and consistent keypoint matches:
$\bZ_{ij} := \hat\bP_{i}{\hat\bP_{j}}^\top$ for any $ij\in E$.
\begin{algorithm}[h]
\caption{MatchFAME}
\begin{algorithmic}\label{alg:matchfame}
\REQUIRE CEMP-Partial output, $\{\hat{s}_{ij}\}_{ij \in E}$, pairwise matching matrix $\bX$, parameter $\gamma$, number of iterations $t_0$
\STATE \textbf{Steps:}
\STATE Form weights on $E$: $w_{ij}=\hat{s}_{ij}$ for $ij \in E$
\STATE Find the MST of the weighted graph $G([n],E,\{w_{ij}\}_{ij \in E})$
\STATE For the root of tree at node $i_0 \in [n]$: $\bP_{i_0}^{(0)} = \bU_{m_i\times m}$
\STATE Propagate along MST: $\bP_j^{(0)} = \textbf{Proj}(\bX_{ji}\bP_i^{(0)})$; if a column of
$[\bP_j^{(0)}]_{j\in [n]}$ is $\boldsymbol{0}$, randomly assign 1 to  one of its elements
\STATE Reset 
graph weights: $w_{ij} = e^{-\gamma s_{ij}}$ for $ij\in E$
\STATE Set $t=0$
\WHILE{$t < t_0$ and $\bP_i^{(t)} \neq \bP_i^{(t-1)}$}
\STATE Compute $\bP_i^{(t+1)}$ by \eqref{eq:WPPM} for $i\in [n]$
\STATE $t=t+1$
\ENDWHILE
\ENSURE $\bZ_{ij} = \bP_i^{(t_0+1)} {\bP_j^{(t_0+1)}}^{\top} $ for $ij \in E$
\end{algorithmic}
\end{algorithm}

The default parameters for CEMP-Partial are $T=25$ and $\beta_t= \min \{1.2^{t},40\}$.
The default parameters for MatchFAME are $t_0 = 60$ and $\gamma = 4$ (for the noiseless synthetic data we use $\gamma =20$ to approach exact recovery).

%




\subsection{Time and Space Complexity}
\label{sec:complexity}

Recall that $n$ is the number of images and $M$ is the total number of 2D keypoints (which is different from $m$), so that the average number of 2D keypoints in each image is $M/n$. Let a$n_E$ denote the number of edges in the viewing graph.
The time and space complexities of CEMP-Partial are  $O(M n_E)$ and $O(n n_E)$ respectively.
Let $\bP$ be the $M \times m$ block matrix whose $i$-th block is $\bP_i$,  $\bX$ be the $M \times M$ block matrix whose $(i,j)$-th block is $\bX_{ij}$. The power iterations in \eqref{eq:WPPM}, considering all $i\in [n]$, can be equivalently viewed as a multiplication between a weighted sparse matrix $\bX$ and a sparse matrix $\bP$.
Using the facts that there are at most $n$ nonzero elements in each column of $\bP$ and each row of $\bX$ and each row of $\bP$ has at most 1 nonzero element, the time and space complexities for the power iterations are $O(nM)$ and $O(nm)$, respectively. Note that PPM requires an additional projection onto the set of partial permutations, whose time complexity is $O(Mm)$ and space complexity $O(nm)$.
The time complexity for finding MST is $O(n_E \log n)$ and its space complexity is $O(n)$. The time complexity of multiplying matrices along the spanning tree is $O(n^2)$ and it requires no additional memory. To sum up, since $n<n_E$, MatchFAME requires  $O(M\cdot\max(m, n_E))$ time and $O(n\cdot\max(m,n_E))$ memory. In comparison, Spectral computes the top $m$ eigenvectors of $\bX$, which is commonly solved by the power method.  This method requires the iterative multiplication of $\bX$ with the updated $M \times m$ dense eigenmatrix, whose time and space complexities are $O(Mmn)$ and $O(Mm)$, respectively, which is much larger than that of MatchFAME.

\section{Theoretical Guarantees for CEMP-Partial}
\label{sec:theory}
We assume the adversarial corruption model for PPS (see  \S\ref{sec:pps}).
We show that the estimated corruption levels by CEMP-Partial at good edges
converge to 0 linearly and uniformly. Moreover,
we show that the estimated corruption levels at the bad edges, can be separated from the ones at the good edges.

\noindent
{\bf Definitions:}
For $ij \in E$,
define
$G_{ij}:=\{k\in [n]: ik, jk \in E_g\}$ and associate $k \in G_{ij}$ with the cycle $ijk$. We refer to the elements of $G_{ij}$ as good cycles (with respect to $ij$).
Let
$$\lambda := 1- \min_{ij} {|G_{ij}|}/{|N_{ij}|}.$$

Some of the following definitions, which lead to the notion of $s_{ij}^*$, are demonstrated in Figure \ref{fig:definition}. For any image $i$, let $I_i = \{p_{i,j}\}_{j=1}^{m_i}$ denote the set of its 2D keypoints.
Let $h : \cup_{i=1}^n I_i \rightarrow [m]$ map 2D keypoints to universal keypoints such that for any $i \in [n]$ and $p_{i,i'} \in I_i$, $h(p_{i,i'})$ is the index of the
3D keypoint, i.e., $P_i^*(i',h(p_{i,i'}))=1$.
For $ij \in E$, let $U_{ij} = h(I_i) \cup h(I_j)$.
We partition $U_{ij}$ into $U_{ij}^{\text{good}}$ and $U_{ij}^{\text{bad}}$. The set $U_{ij}^{\text{good}}$ contains the keypoints in $U_{ij}$ that match the ground-truth keypoints or match no keypoint when no ground-truth matching exists.
That is, $k \in h(I_i) \cap h(I_j)$ is in $U_{ij}^{\text{good}}$ whenever there exists $p_{i,a} \in I_i$ and $p_{j,b} \in I_j$ such that $\bP_i^*(a,k) = \bP_j^*(b,k) = \bX_{ij}(a,b) = 1$. Furthermore, $k \in h(I_i) \setminus h(I_j)$ (or $k \in h(I_j) \setminus h(I_i)$) is in $U_{ij}^{\text{good}}$ whenever there are no $p_{i,a} \in I_i$ and $p_{j,b} \in I_j$ such that $\bP_i^*(a,k) = \bX_{ij}(a,b) = 1$
(or $\bP_j^*(b,k) = \bX_{ij}(a,b) = 1$).
Similarly, $U_{ij}^{\text{bad}} := U_{ij}\setminus U_{ij}^{\text{good}}$ is the set of keypoints in $U_{ij}$ that match wrong keypoints in the other image, or match no keypoint if a ground-truth match exists.
For any $ij \in E$, define $m_{ij} = |U_{ij}|$, $m_{ij}^{\text{bad}} = |U_{ij}^{\text{bad}}|$ and note that  $s_{ij}^* = {m_{ij}^{\text{bad}}}/{m_{ij}}$.

\begin{figure}
    \centering
    \includegraphics[width=8cm]{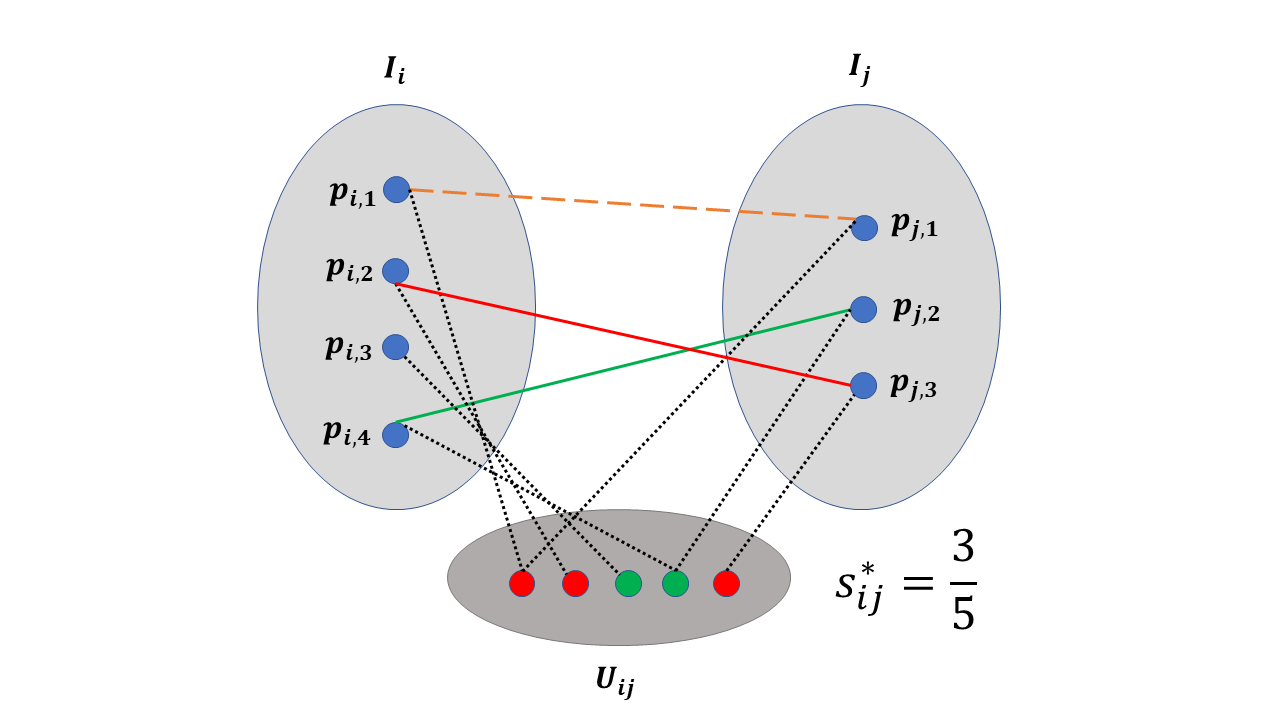}
    \caption{An illustration of $s_{ij}^*$ and related definitions. The sets $I_i$ and $I_j$ are of 2D keypoints in images $i$ and $j$, respectively, and the set $U_{ij}$ is of the corresponding universal keypoints in images $i$ and $j$. Green lines denote the good keypoint matches, red lines denote the bad keypoint matches, dashed orange lines denote the missing keypoint matches, and black dotted lines denote the correspondence between 2D and 3D keypoints (i.e., the $h$ function). Green dots in $U_{ij}$ represent elements of $U_{ij}^{\text{good}}$ and red dots in $U_{ij}$ represent elements of $U_{ij}^{\text{bad}}$. We note that $m_{ij} = 5$, $m_{ij}^{\text{good}} = 2$, $m_{ij}^{\text{bad}} = 3$ and $s_{ij}^* = \frac{3}{5}$.}
    \label{fig:definition}
\end{figure}

In order to guarantee that the separation problem is well-posed, we need to ensure two different types of conditions. The first is that there are sufficiently many good cycles. We ensure this condition by bounding $\lambda$ (similarly to \cite{cemp}). The second is that sufficiently many good matches exist. Indeed, since partial permutations can be very sparse and even zero matrices such a condition is necessary. For this purpose, we formulate the following cycle-verfiability condition (we further interpret it and clarify its name in the supplemental material). It uses a parameter $p_v$ that expresses the proportion of verifiability.
\begin{definition}
\label{def:cycle_verifiable}
Given $p_v\in (0, 1]$, a graph $(V,E)$ is $p_v$-cycle verifiable if for any $ij \in E$
there are at least $p_v |G_{ij}|$ good cycles w.r.t.~$ij$
such that for each such cycle, $ijk$, the following property holds: if $a  \in I_i \cup I_j$, then
there exists $b \in I_k$
that matches $a$ (i.e., if $a = p_{i,r}$
and $b = p_{k,t}$, then $\bX_{ik}(r,t) = 1$).
\end{definition}

\noindent
{\bf Formulation of the Main Theorem:}
\begin{theorem}\label{thm:main}
If $G$ 
is $p_v$-cycle verifiable
and $s_{ij}^{(t)}$ is computed by CEMP-Partial
with $\beta_0 \le \frac{1}{2\lambda}$ and $\beta_{t+1} = r\beta_t$, where
$ \lambda < 1+\frac{3em}{p_v}-\sqrt{\frac{3em}{p_v}(2+\frac{3em}{p_v})}$ and $1 < r < \frac{(1-\lambda)^2p_v}{6em\lambda}$, then  $\forall t>0$ 
$$\forall ij \in E_g \ \ s_{ij}^{(t)} \le \frac{1}{2 \beta_0 r^t} \ \text{ and } \ \forall ij \in E_b \ \ s_{ij}^{(t)} \ge \frac{p_v}{3e}(1-\lambda)s_{ij}^*.$$
\end{theorem}

Theorem \ref{thm:main} guarantees exact separation between clean and corrupted edges in a worst-case scenario setting. Unlike previous PPS works \cite{deepti, chen_partial}, our theory is completely deterministic and does not rely on the assumptions of the underlying distribution of partial permutations. Its deterministic conditions guarantee well-posedness of the separation problem.

\section{Numerical Experiments}
\subsection{Synthetic Data Experiments}
\label{sec:synthetic}
We test MatchFAME, MatchEIG \cite{MatchEig}, Spectral  \cite{deepti}, and PPM \cite{Chen_PPM} on synthetic datasets generated by two different models described in \S\ref{sec:lbc_lac} and \S\ref{sec:ucm}. The underlying graph in both cases is generated by an Erd\"{o}s-R\'{e}nyi model, $G(n,p)$, with probability of edge connection $p=0.5$.
The inclusion of a keypoint in an image $I$ is an independent event from the rest of the keypoints that occurs with probability $p_I=0.8$.
In both cases, the number of images is $n=100$ and the universe size is $m = 20$ (except for the runtime  experiments in \S\ref{sec:ucm}). Since $m$ is typically unknown in practice, all algorithms use the following estimate for it: $\hat m = 2\lceil {M}/{n} \rceil$. In order to demonstrate near exact separation, we use $\gamma = 20$ for the synthetic data, where the rest of the default parameters are as specified in \S\ref{sec:matchfame}.

After generating the graph $G([n],E)$ by the $G(n,p)$ model, we also randomly generate an $n  m \times  m$ ground truth image-universe matching matrix $\bP^{\text{full}}$. Each of its $n$ blocks is obtained by randomly permuting columns of $\bI_{m \times m}$.
We further compute $\bX^{\text{full}} = \bP^{\text{full}} {\bP^{\text{full}}}^\top$.
The models below 
further corrupt $\bX^{\text{full}}$ and result in modified matrix $\hat{\bX}^{\text{full}}$.
We then generate the  keypoint indices, $\{I_i\}_{i=1}^n$, as follows: For each $i \in [n]$, we independently generate $m$ i.i.d. $\sim B(1,0.8)$ random variables and let $I_i$ be the set of indices of random variables with output 1.
For each $i \in [n]$ and the $i$th $m\times m$ block of $\bP^{\text{full}}$, we keep the rows with indices in $I_i$ and discard the rest. This results in a block $\bP_i^*$ of size  $m_i \times m$. The resulting block matrix of all modified blocks is of size $M \times m$
and denoted by  $\bP^*$. We further set $\bX^* = \bP^{*} \bP^{*T}$. For the $ij$-th $m\times m$ block of $\hat \bX^{\text{full}}$, we keep its rows that appear in $I_i$ and columns that appear in $I_j$ and discard the rest. This results in a block of size $m_i \times m_j$. We stack these blocks to form the $M \times M$ matrix $\bX$.

Let $\odot$ denote the element-wise product and $\bZ_{ij}$ denote the $ij$-th block of the output $\bZ$.
Precision and recall were respectively computed as follows:
$$\sum_{ij \in E_b} \langle \bX_{ij}^* \odot \bX_{ij}, \bZ_{ij} \odot \bX_{ij} \rangle \slash \sum_{ij \in E_b} \| \bZ_{ij} \odot \bX_{ij} \|_F^2 \ \text{ and }$$
$${\sum_{ij \in E_b} \langle \bX_{ij}^* \odot \bX_{ij}, \bZ_{ij} \odot \bX_{ij} \rangle} \slash {\sum_{ij \in E_b} \| \bX_{ij}^* \odot \bX_{ij} \|_F^2}.$$

Since in SfM precision is more important than recall \cite{MatchEig}, we find an algorithm superior to another if it achieves significantly higher precision with almost equal or higher recall.

\subsubsection{Data Generated by the LBC and LAC Models}
\label{sec:lbc_lac}
We extend the Local Biased Corruption (LBC) and Local Adversarial Corruption (LAC) models of \cite{IRGCL} to PPS. Both models introduce nonuniform corruption concentrated in some clusters, where for some nodes, most of their neighboring edges are corrupted.
LBC assumes that bad edges are also cycle-consistent, so they behave like good edges. In addition, it uses a sample-rejection procedure so that the distribution of the corrupted partial permutations deviates from the uniform distribution. LAC is even more malicious, it corrupts edges in a way that would seem the absolute permutations of the selected nodes are perturbed versions of  $\bI_{m_i,m}$.

We let $\{\bP_i^c\}_{i \in [n]}$ be i.i.d.~sampled from the Haar measure on $\mathcal{P}_{\text{full}}^{m \times m}$. Starting from ${\hat \bX}^{\text{full}} = \bX^{\text{full}}$, for both models we independently sample $n_c$ nodes as our corruption seed nodes. For each corruption seed node, we independently corrupt its associated edges with probability $0.9$ for LBC and $0.6$ for LAC. Each $ij \in E_b$ in LBC is corrupted as follows:
$$\hat \bX_{ij}^{\text{full}} = \begin{cases}
\hat \bX_{ij}^{\text{full}} \sim \text{Haar}(\mathcal{P}_{\text{full}}^{m \times m}), & \text{if } \langle \bP_i^c {\bP_j^c}^\top, \bX_{ij}^{\text{full}} \rangle > 1;\\
\bP_{i}^c {\bP_j^{c}}^\top, & \text{otherwise}.
\end{cases}$$
In LAC, for $ij \in E_b$,
$\hat{\bX}_{ij}^{\text{full}} = \bQ_{ij}^c {\bP_{j}^{\text{full}}}^\top$, where $\bQ_{ij}^c$ is obtained by randomly permuting 3 of the columns of $\bI_{m \times m}$.


The resulting precision and recall are reported in Figure \ref{fig:LBC} for the LBC model and Figure \ref{fig:LAC} for the LAC model. We note that for both models, MatchFAME recovers almost exactly all bad edges, where other algorithms have lower errors and are not sufficiently close to exact recovery. MatchFAME also obtains the highest recall scores, which are close to 1.

\begin{figure}%
    \centering
    \includegraphics[width=4.1cm]{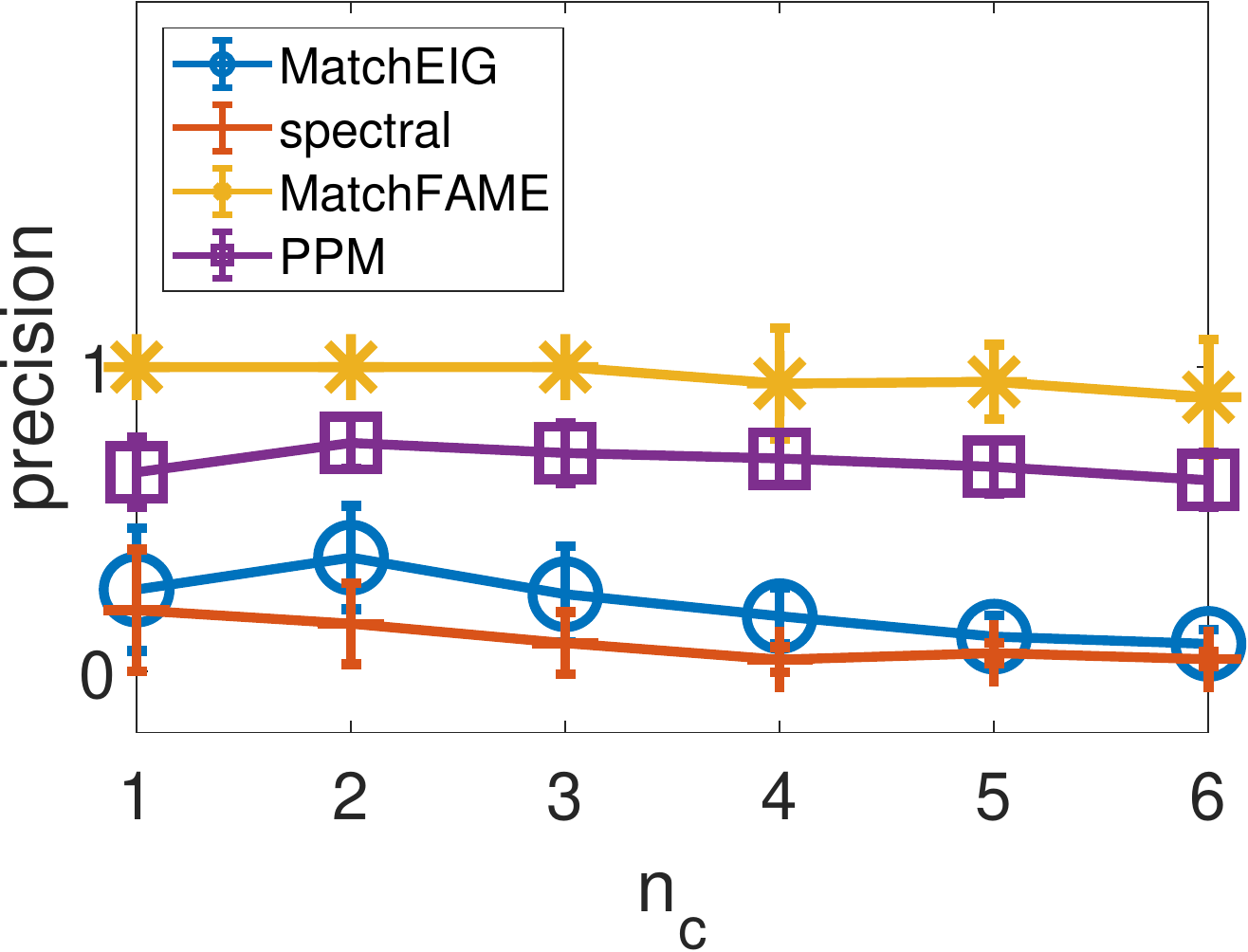} %
    \hfill
    \includegraphics[width=4.1cm]{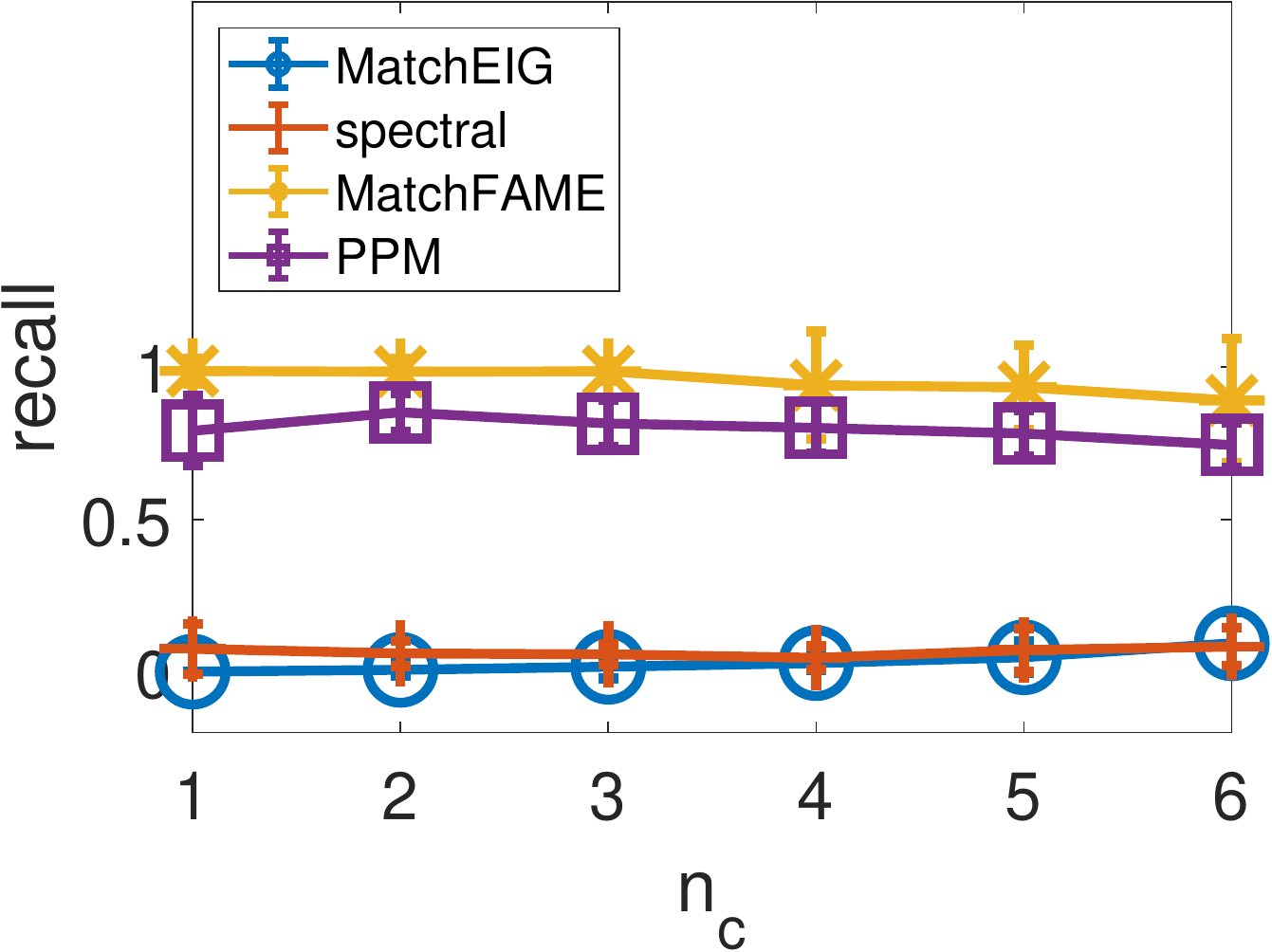} %
    \caption{Average precision and recall for the LBC model}%
    \label{fig:LBC}%
\end{figure}

\begin{figure}%
    \centering
    \includegraphics[width=4.1cm]{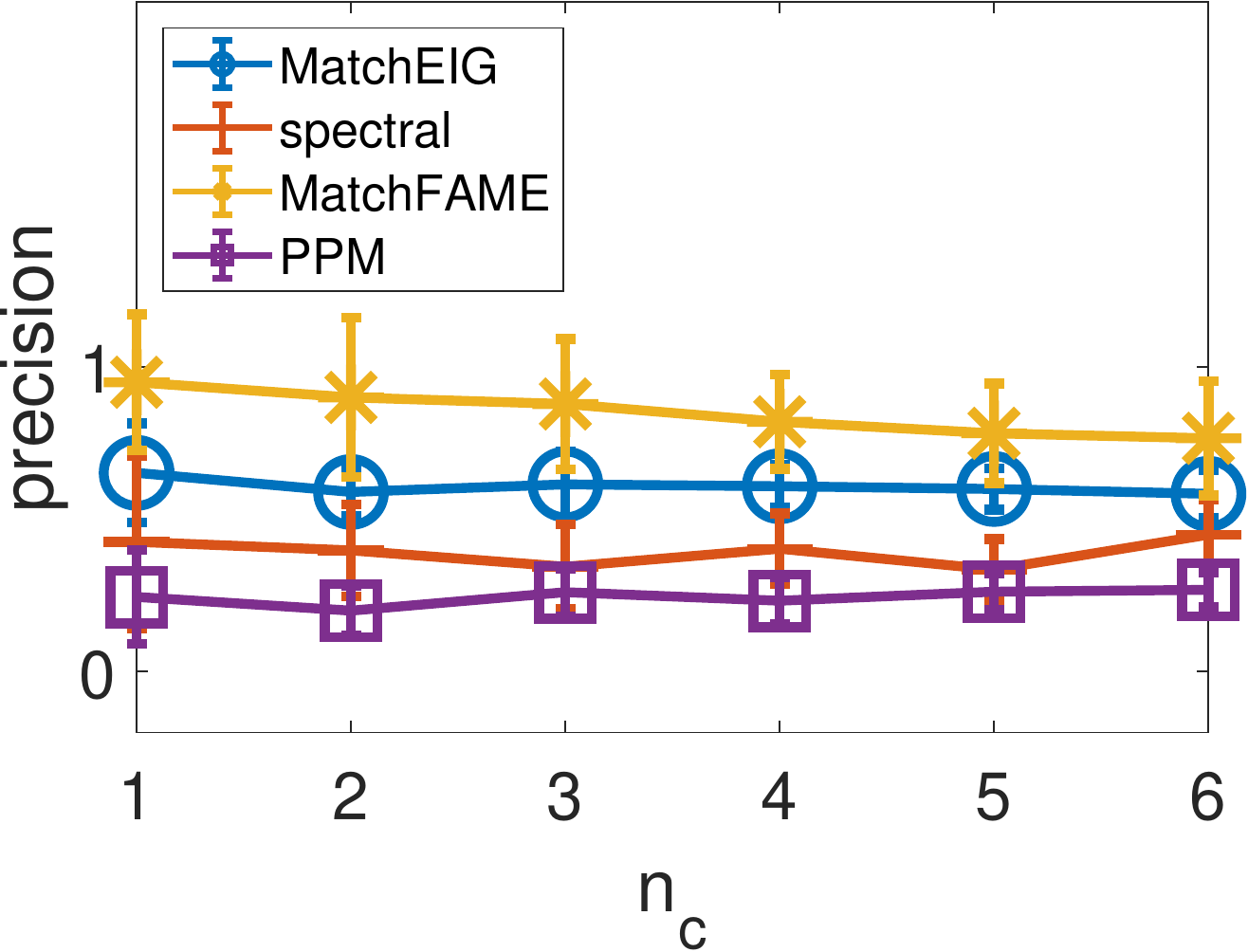} %
    \hfill
    \includegraphics[width=4.1cm]{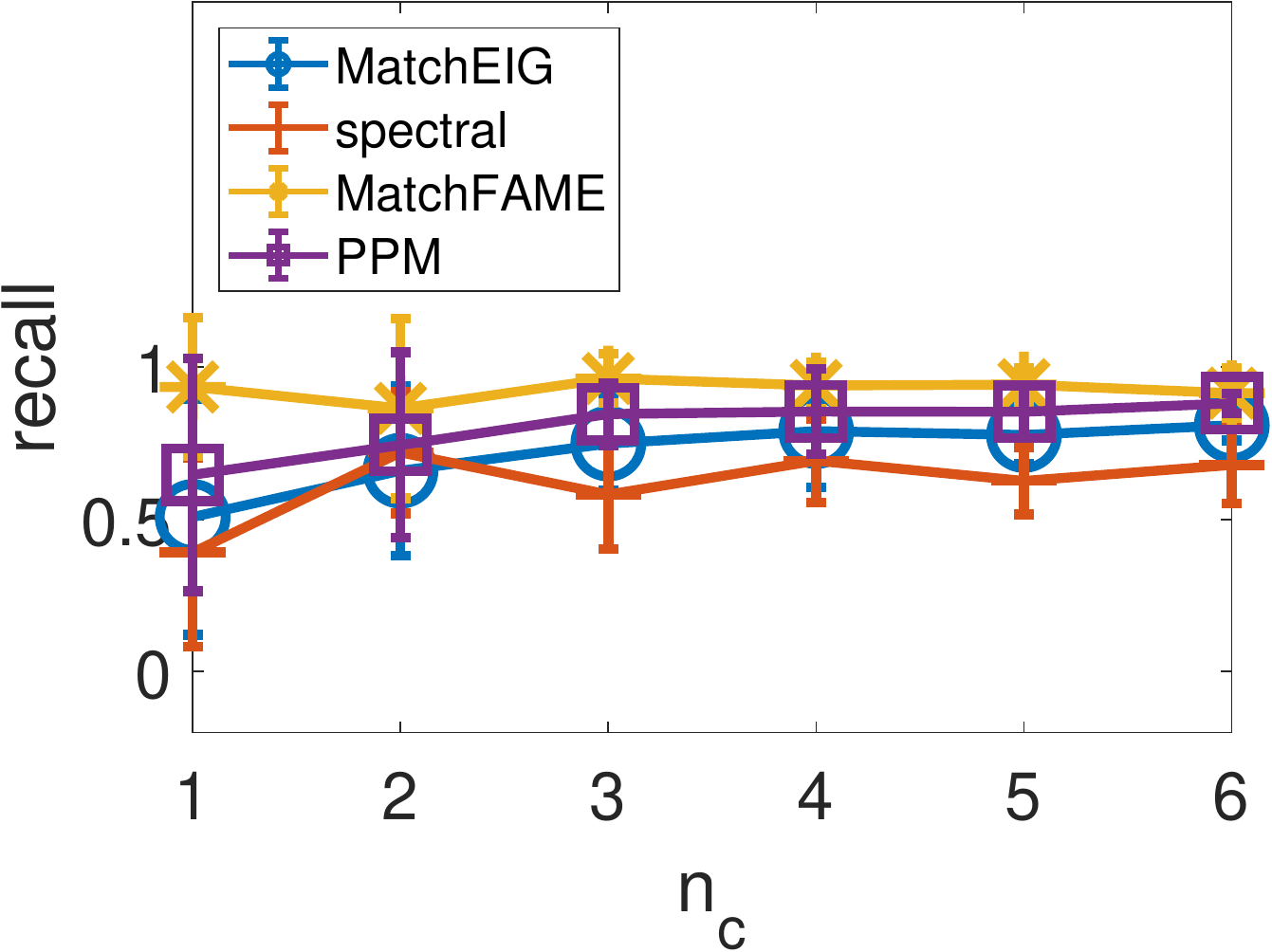} %
    \caption{Average precision and recall for the LAC model}%
    \label{fig:LAC}%
\end{figure}





\subsubsection{Data Generated by the Uniform Corruption Model}
\label{sec:ucm}
We test MatchFAME and competing models on data generated from the uniform corruption model (UCM). In this model, an edge
is randomly selected with probability $q$ and then corrupted as follows: $\hat \bX_{ij}^{\text{full}} \sim \text{Haar}(\mathcal{P}_{\text{full}}^{m \times m}) $.
For unselected edges: $\hat \bX_{ij}^{\text{full}}=  \bX_{ij}^{\text{full}}$.
We let $q$ range between $0.5$ and $0.9$.  Figure \ref{fig:unif2} reports the precision and recall of the different methods. We observe that for UCM, the precision of MatchFAME decreases the slowest among all algorithms. We observe that MatchFAME and PPM have similar high recall, while MatchEIG and Spectral have relatively lower recall.

\begin{figure}%
    \centering
    \includegraphics[width=4.1cm]{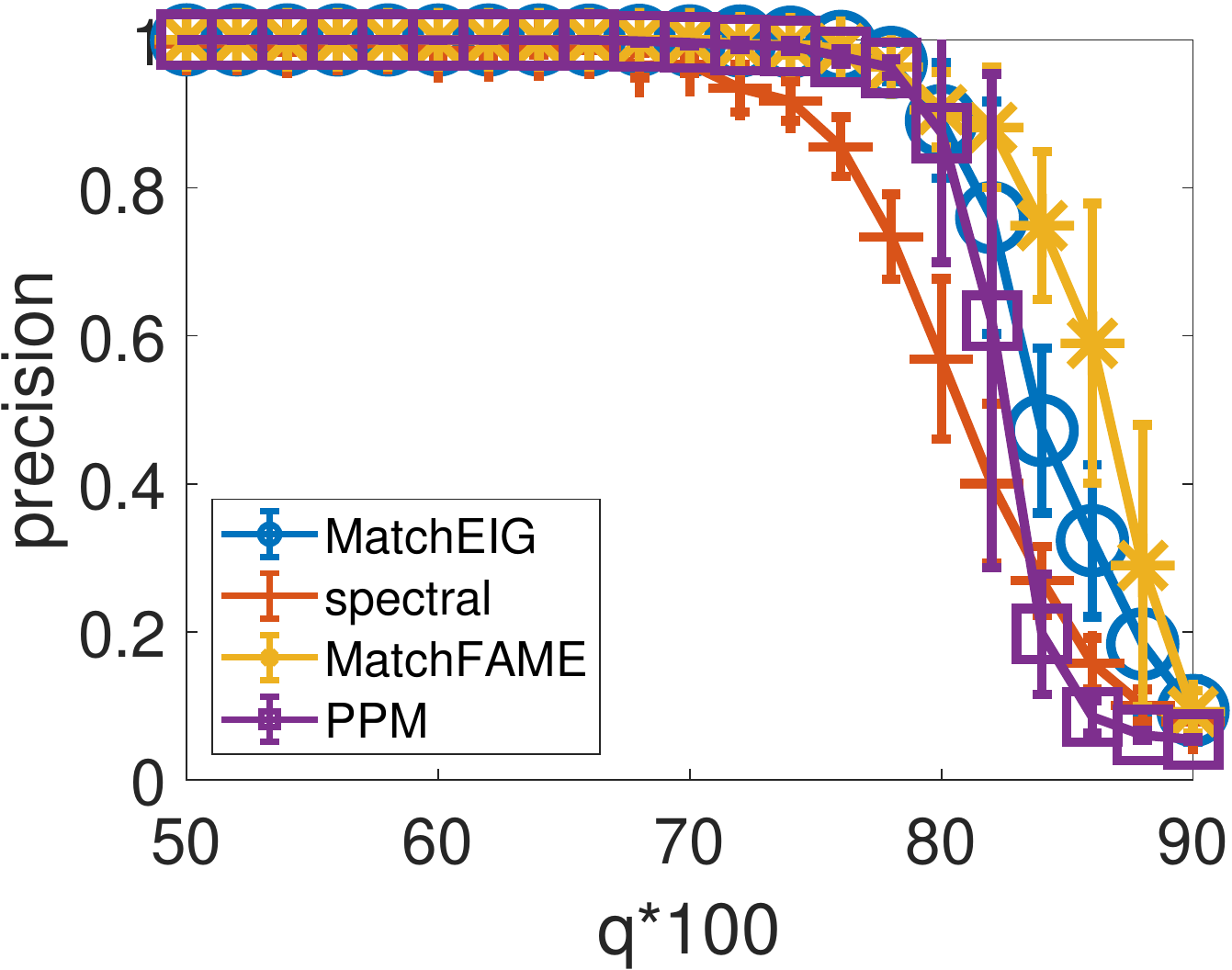} %
    \includegraphics[width=4.1cm]{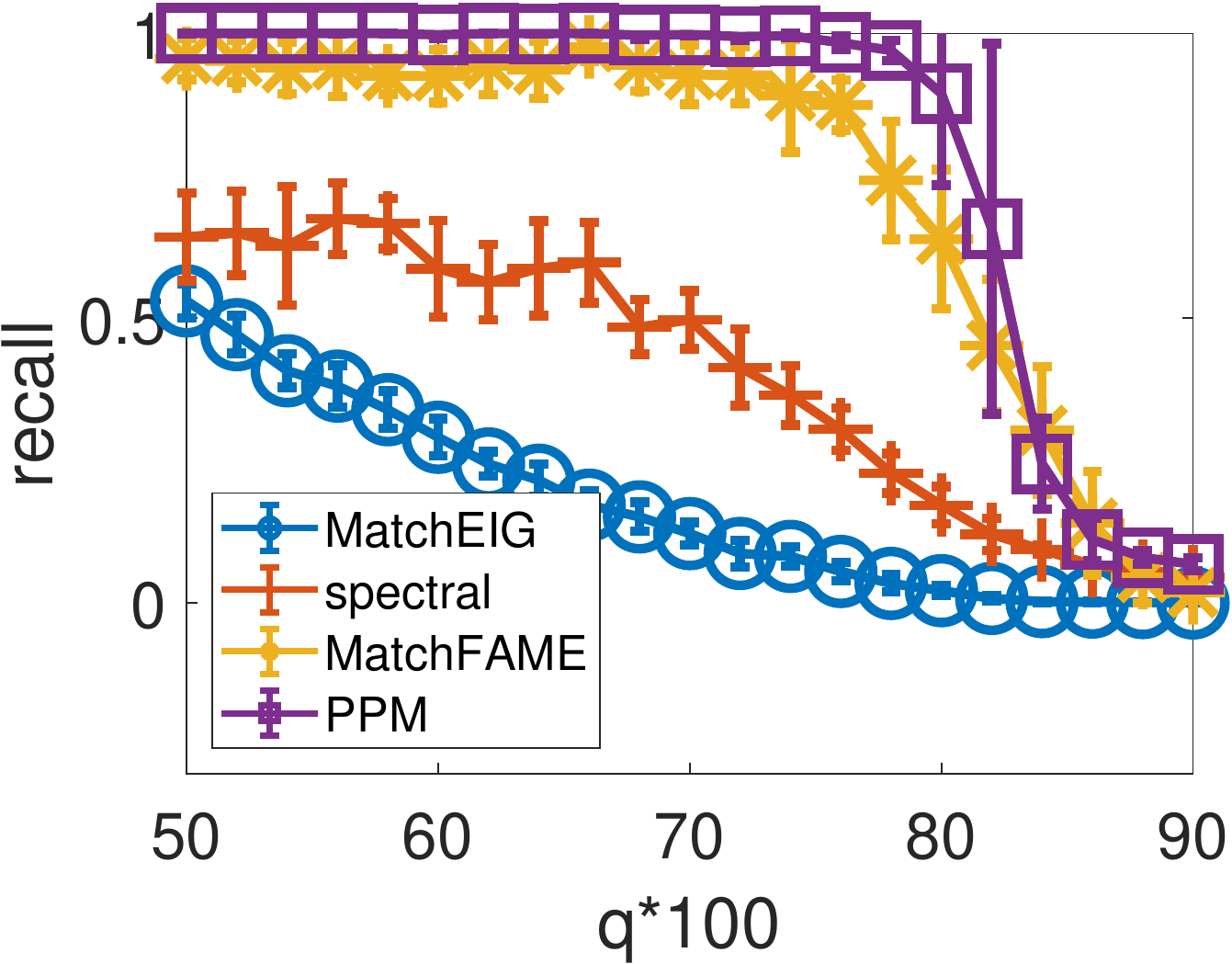} %
    \caption{Average precision and recall for UCM}%
    \label{fig:unif2}%
\end{figure}

\begin{table}
\centering
\resizebox{0.78\columnwidth}{!}{
\begin{tabular}{|c|c||c||c|}
    \hline
    Algorithm &\multicolumn{1}{c||}{$m=500$} &\multicolumn{1}{c||}{$m=1000$}
     &\multicolumn{1}{c|}{$m=2000$} \\ \hline
    MatchFAME & 1 & 2 & 27 \\ \hline
    MatchEIG & 19 & 194 & 1297 \\ \hline
\end{tabular}}
\caption{Runtime (seconds) for UCM with $n = 20$. \label{table:ucm_n}}
\end{table}

\begin{table}
\centering
\resizebox{0.78\columnwidth}{!}{
\begin{tabular}{|c|c||c||c|}
    \hline
     Algorithm &{$n=300$} &{$n=500$}
     &{$n=700$} \\ \hline
    MatchFAME & 86 & 304 & 826 \\ \hline
    MatchEIG & 116 & 1099 & $>5000$ \\ \hline
\end{tabular}}
\caption{Runtime (seconds) for UCM with $m = 20$. \label{table:ucm_m}}
\end{table}

For runtime comparison, we compared MatchFAME with MatchEIG, which is the fastest PPS method.
Table \ref{table:ucm_n} fixes $n=20$ and reports results for different values of $m$ and Table \ref{table:ucm_m} fixes $m=20$ and reports results for different values of $n$ (experiments stopped when the time was larger than 5000 seconds). Clearly, MatchFAME is significantly faster than MatchEIG for large $m$ and sufficiently large $n$.



\begin{table*}[t]
\centering 
\resizebox{1.65\columnwidth}{!}{
\renewcommand{\arraystretch}{1.3}
\tabcolsep=0.1cm
\begin{tabular}{|l||c|c||c|c|c|c|c|c||c|c|c|c|c|c|c|}
\hline
Algorithms & \multicolumn{2}{c||}{}& \multicolumn{6}{c||}{LUD} &
\multicolumn{7}{c|}{MatchFAME+LUD}
  \\\hline
\text{Dataset}& $n$ & $M$ & $n_{\text{LUD}}$  & {\large $\hat{e}_R$} & {\large$\tilde{e}_R$} & {\large $\hat{e}_T$} & {\large$\tilde{e}_T$}   & {$T_{\text{total,LUD}}$} & {$n_{\text{FAME}}$} &
{\large $\hat{e}_R$} & {\large$\tilde{e}_R$} & {\large $\hat{e}_T$} & {\large$\tilde{e}_T$} &  $T_{\text{FAME}}$ &  $T_{\text{total}}$
\\\hline
Alamo & 570 & 606963
& 557
& 20.81 & 16.88 & 8.00 & 5.23 & 7945.9
& 494
& \textbf{17.50} & \textbf{14.19} & \textbf{6.99} & \textbf{4.58} & 6272.0 & 16523.0
\\\hline
Ellis Island & 230 & 178324
& 223
& 2.14 & 1.15 & 22.99 & 22.82 & 1839.2
& 216
& \textbf{1.79} & \textbf{1.02} & \textbf{22.63} & \textbf{21.57} & 341.5 & 2115.7
\\\hline
Gendarmenmarkt & 671 & 338800
& 652
& \textbf{40.14} & 9.30 & \textbf{38.55} & \textbf{18.33} & 3527.7
& 574
& 40.20 & \textbf{9.20} & 42.29 & 21.65 & 662.2 & 3960.2
\\\hline
Madrid Metropolis & 330 & 187790
& 315
& 13.49 & 9.58 & 14.10 & 6.81 & 1579.8
& 266
& \textbf{10.21} & \textbf{6.18} & \textbf{9.90} & \textbf{4.42} & 179.7 & 1663.2
\\\hline
Montreal N.D.& 445 & 643938
& 439
& 2.55 & 1.06 & 1.51 & 0.66 & 5078.9
& 389
& \textbf{1.45} & \textbf{0.79} & \textbf{1.16} & \textbf{0.60} & 4021.0 & 8666.4
\\\hline
Notre Dame & 547 & 1345766
& 545
& 3.72 & \textbf{1.44} & 1.45 & 0.41 & 11315.2
& 526
&\textbf{3.19} & 1.48 & \textbf{1.27} & \textbf{0.40} & 27222.4 & 39189.5
\\\hline
NYC Library&  313 & 259302
& 306
& 3.99 & 2.14 & 6.89 & 2.72 & 1495.9
& 270
& \textbf{2.96} & \textbf{1.94} & \textbf{6.40} & \textbf{2.43} & 219.9 & 2112.5
\\\hline
Piazza Del Popolo & 307 & 157971
& 300
& 7.04 & 4.05 & 6.68 & 2.33 & 1989.9
& 241
& \textbf{1.39} & \textbf{0.83} & \textbf{2.26} & \textbf{1.33} & 309.6 & 2519.3
\\\hline
Piccadilly& 2226 &  1278612
& 2015
& 8.05 & 3.82 & 5.47 & 2.98 & 21903.3
& 1479
& \textbf{4.96} & \textbf{2.86} & \textbf{3.96} & \textbf{2.16} & 26170.1 & 48157.3
\\\hline
Roman Forum&  995 & 890945
& 971
& 6.64 & 5.02& 12.67 & 5.60 & 4858.0
& 733
& \textbf{5.63} & \textbf{4.29} & \textbf{11.73} & \textbf{5.55} & 1548.3 & 6955.3
\\\hline
Tower of London& 440 & 474171
& 431
& \textbf{6.89} & 4.29 & 21.47 & 6.85 & 1759.2
& 359
& 7.13 & \textbf{4.19} & \textbf{13.54} & \textbf{6.32} & 282.5 & 2040.4
\\\hline
Union Square& 733 &  323933
& 663
& 10.40 & 6.70 & 15.27 & 11.14 & 1950.6
& 451
& \textbf{8.26} & \textbf{5.29} & \textbf{10.71} & \textbf{8.66} & 224.9 & 2064.9
\\\hline
Vienna Cathedral&  789 &  1361659
& 758
& 6.45 & 3.10 & 14.18 & 8.12 & 10866.0
& 631
& \textbf{4.21} & \textbf{2.02} & \textbf{11.90} & \textbf{6.90} & 21082.7 & 30066.5
\\\hline
Yorkminster & 412 &  525592
& 407
& \textbf{4.25} & 2.71 & 6.45 & 3.68 & 2267.3
& 359
& 4.26 & \textbf{2.56} & \textbf{6.41} & \textbf{3.37} & 586.4 & 2936.5
\\\hline
\end{tabular}}
\caption{Performance on the Photo Tourism database: $n$ and $M$ are the number of nodes and keypoints, respectively; $n_{\text{LUD}}$ and $n_{\text{FAME}}$ are the remaining number of cameras after the LUD pipeline and our pipeline, respectively; $\hat e_R$ $\tilde e_R$ indicate mean and median errors of absolute camera rotations in degrees, respectively; $\hat e_T$ $\tilde e_T$ indicate mean and median errors of absolute camera translations in meters, respectively; $T_{\text{FAME}}$, $T_{\text{total,LUD}}$ and $T_{\text{total}}$ are the runtime of MatchFAME, the total runtime of LUD pipeline and the total runtime of our pipeline, respectively (in seconds).}
\label{tab:real}
\end{table*}

\subsection{Real Data Experiments}
\label{sec:real}


We test MatchFAME on the Photo Tourism dataset \cite{1dsfm14}. This large-scale dataset contains 14 sets of images for stereo reconstruction. The number of images in each dataset ranges from 230 to 2226. Given initial keypoint matches obtained by \cite{SenguptaAGGJSB17}, we form our pairwise matching matrix and estimate the universe size with $\hat m = 16\lceil M/{n} \rceil$. We apply  MatchFAME with its default parameters.
After getting the output keypoint matching, we use RANSAC to estimate the fundamental matrices for each edge. If an edge has less than 16 remaining keypoint matches, we remove this edge. Then we decompose the resulting fundamental matrices and feed the estimated rotations and translations to the LUD \cite{ozyesil2015robust} camera pose solver to obtain the final estimate of absolute rotations and translations. Note that LUD extracts the largest parallel rigid component of the remaining graph, therefore removing edges can cause loss of cameras. We compare the average and median translation error, average and median rotation error, and runtime to the original pipeline of LUD.
We remark that we did not compare with other PPS algorithms as the ones that were available at the time of the submission were not scalable and could not handle the Photo Tourism dataset.

Table \ref{tab:real} reports results for both the LUD pipeline and the incorporation of MatchFAME within the LUD pipeline. We consider improvement over the LUD result when we obtain a smaller error on at least 3 of the 4 error statistics.
MatchFAME is successful in improving the estimates of translation and rotation without significant loss of cameras in 13 of the 14 datasets. The most significant improvement is on Piazza Del Popolo, where our mean and median rotation error decreased by $80.3\%$ and $79.5\%$ respectively. For most other datasets, the improvement was not marginal. Indeed, in 12 of these 13 datasets, at least one error statistic decreased by more than $15\%$. The only dataset we didn't improve is Gendarmenmarkt, which has very high error because of its symmetric buildings.
MatchFAME removes at most $32.0\%$ of the total cameras. The remaining cameras are sufficient for 3D reconstruction since the Photo Tourism cameras are sampled densely. Therefore our pipeline will not cause significant loss of quality of 3D reconstruction. That is, MatchFAME is able to remove cameras with erroneous keypoints without losing the stereo reconstruction power.

On larger datasets, the total time of MatchFAME is 2 to 3 times larger than that of the original LUD pipeline. We thus find MatchFAME scalable. On smaller datasets, such as Union Square and Madrid Metropolis, MatchFAME consumes much less time in comparison to the original LUD pipeline.

\section{Conclusion}
We develop MatchFAME, a robust, fast, accurate and memory efficient PPS method.
For this purpose we first developed CEMP-Partial for corruption estimation in PPS and theoretically guaranteed it under adversarial corruption. In doing this, we were able to overcome nontrivial challenges of the PPS problem. We also proposed an efficient weighted PPM method that utilizes the output of CEMP-Partial and, in particular, does not require spectral initialization. MatchFAME overcomes the three major challenges of PPS: nonuniform corruption, uneven dimensions and sparsity levels, and the large problem size. Synthetic and real data experiments demonstrate the superior precision of MatchFAME over existing standard methods and its scalability to large datasets due to sparse matrix operations. Our method also has some limitations. For example, our MST initialization is quite heuristic and may produce very sparse initialization. Instead, one may consider using minimum-$K$ spanning trees and initialize the solutions by aggregating the initialization from the different spanning trees. Moreover, our weights for PPM depend on a parameter $\gamma$ and we plan to explore in future work the optimal assignment of this parameter, using the estimates of the corruption levels. Furthermore, our theory is currently limited to CEMP-Partial and to adversarial corruption with certain assumptions. Nevertheless, it demonstrates how to handle some challenges that are unique to PPS. We plan to further extend it. We will first explore other corruption models, such as UCM, where we expect stronger convergence guarantees. We also plan to further develop theory for PPM, in particular, for UCM, and hopefully establish a more complete theory for MatchFAME.

\subsection*{Acknowledgement}
This work was supported by NSF awards 1821266, 2124913.

{\small
\bibliographystyle{ieee_fullname}
\bibliography{egbib}
}

\newpage
\appendix

{
\begin{center}
\huge{Supplemental Material}
\end{center}
}

\section{Additional Experiments on the EPFL dataset}

We test MatchFAME on the 6 EPFL datasets following the experimental setup of \cite{MatchEig}. Each dataset includes 8 to 30 images, unlike the large number of images in the Photo Tourism datasets. Given each dataset, we generate and refine the initial keypoint matches with the same procedure introduced in \cite{MatchEig}. We follow their convention and estimate the universe size with $\hat m = 2\lceil M/{n} \rceil$. We implement MatchFAME with its default parameters, though with two changes described below. Indeed, the EPFL dataset contains a lot of noisy edges and thus the weights produced by PPM within the original MatchFAME algorithm are often small. Furthermore, note that $\mathbf{Proj}$ in \eqref{eq:WPPM} is not scale invariant and that the resulting small weights may lead to overly sparse refined matches. Therefore, we slightly changed the implementation of MatchFAME to overcome this issue. First, in order to obtain a dense initialization of partial permutations using MST, instead of assigning 1 to a random element for each zero column, we assign 1 to a random element for each zero row. Since the number of rows is larger than the number of columns, this modification results in a denser initialization of $[\bP_j^{(0)}]_{j\in [n]}$ than that of the original MatchFAME. Second, to make sure that the final output is also sufficiently dense, we drop the step of the weights' normalization within the PPM iterations, which is described below \eqref{eq:WPPM} (this will increase the overall scale of the edge weights and thus the projected matrix is expected to be denser). We remark that these two changes help alleviate the over-sparseness of the final output and ends up with a higher ratio between the number of refined matches and the number of initial matches, which we denote by $\# M$.

In addition to this version of MatchFAME, we also test Spectral, MatchEIG and MatchALS with the same setting as \cite{MatchEig}. Note that the 'ground truth' is obtained by estimating the projection distance of key points on the epipolar line instead of labeling by hand. Therefore the recall score is not a good benchmark on real data. We thus only report the resulting precision, number of remaining edges and runtime in Table \ref{tab:real2}.

MatchFAME achieves the highest precision of all methods in all datasets. Observing $\# M$, we note that MatchFAME has around $20\%$ fewer matches remaining compared to all algorithms, but as long as there are enough matches for each edge, one can reliably compute relative rotations and translations for SfM tasks. We believe removing around $20\%$ more matches is not an essential drawback. Furthermore, MatchFAME is faster than the other methods. In conclusion, MatchFAME can achieve a reasonable estimate of matches within a significant short amount of time.

\begin{table*}[t]
\centering 
\resizebox{1.6\columnwidth}{!}{
\renewcommand{\arraystretch}{1.6}
\tabcolsep=0.15cm
\begin{tabular}{|l||c|c||c||c|c|c||c|c|c||c|c|c||c|c|c|
|c|c|c|}
\hline
Algorithms & \multicolumn{2}{c||}{ }& \multicolumn{1}{c||}{\multirow{ 2}{*}{Initial}} &
\multicolumn{3}{c||}{\multirow{ 2}{*}{MatchEig}} & \multicolumn{3}{c||}{\multirow{ 2}{*}{Spectral}} &
\multicolumn{3}{c||}{\multirow{ 2}{*}{MatchALS}} &
\multicolumn{3}{c|}{\multirow{ 2}{*}{PPM}} &
\multicolumn{3}{c|}{\multirow{ 2}{*}{MatchFAME}}
  \\
 \text{Dataset} &  \multicolumn{2}{c||}{ } &  \multicolumn{1}{c||}{ } &  \multicolumn{3}{c||}{ } & \multicolumn{3}{c||}{ } & \multicolumn{3}{c||}{ } &  \multicolumn{3}{c|}{ }
 &  \multicolumn{3}{c|}{ (ours) }
\\
  \hline
& $n$ & $\hat m                 $ & PR & PR  &  \#M & T  & PR  &  \#M & T& PR  &  \#M & T& PR  &  \#M & T
& PR  &  \#M & T

\\\hline

Herz-Jesu-P25 & 25 & 517 & 89.6 & 94.2 & 73 & 72 & 92.2 & 81 & 125 & 93.3 & 83 & 9199 & 92.5 & \textbf{88} & 125
& \textbf{95.0} & 78 & \textbf{15}

\\\hline

Herz-Jesu-P8 & 8 & 386 & 94.3 & 95.2 & \textbf{97} & \textbf{1} & 95.3 & 92 & 4 & 95.9 & 76 & 155  & 95.4 & 94 & 5 
& \textbf{95.9} & 83 & 3

\\\hline

Castle-P30 & 30 & 445 & 71.8 & 84.7 & 55 & 64 & 80.6 & 72 & 99 &  80.4 & 76 & 13583 & 80.2 & \textbf{77} & 112
& \textbf{87.9} & 61 & \textbf{15}

\\\hline

Castle-P19 & 19 & 314 & 70.1 & 79.7 & 57 & 23 & 76.3 & \textbf{76} & 21 & 77.0 & 74 & 1263 & 77.5 & \textbf{76} & 33
& \textbf{83.0} & 56 & \textbf{4}

\\\hline

Entry-P10 & 10 & 432 & 75.4 & 79.9 & 78 & 11 & 82.1 & 78 & 30 & 77.3 & 77 & 322 & 80.7 & \textbf{83} & 34
& \textbf{83.1} & 69 & \textbf{5}

\\\hline

Fountain-P11 & 11 & 374 & 94.2 & 95.4 & 81 & 8 & 95.4 & 93 & 14 & 95.7 & 82 & 333 & 95.6 & \textbf{94} & 18
& \textbf{96.7} & 81 & \textbf{5}

\\\hline

\end{tabular}
}

\caption{Performance on the EPFL datasets. $n$ is the number of cameras; $\hat m$, the approximated $m$, is twice the averaged $m_i$ over $i\in [n]$; PR refers to the precision $|\hat E \cap E_g|/|\hat E|$, which is expressed in percentage (the higher the better); \#M is the ratio (expressed in percentage) between the number of refined matches and the number of initial matches; $T$ is runtime in seconds. }

\label{tab:real2}
\end{table*}

\section{Clarifications}

We clarify some definitions and expand on various claims mentioned in the paper.

\subsection{More on Cycle Consistency and Inconsistency}

We referred to a cycle $ijk$ as consistent whenever $\bX_{ij}\bX_{jk}\leq \bX_{ik}$, $\bX_{jk}\bX_{ki}\leq \bX_{ji}$ and $\bX_{ki}\bX_{ij}\leq \bX_{kj}$.
Note that $\bX_{ij}\bX_{jk}$ is a binary matrix with ones whenever there are paths of lengths 2 between keypoints of images $i$ and $k$ and $\bX_{ik}$ is binary matrix with ones whenever there are paths of lengths 1 (single edges) between keypoints of image $i$ and $k$. That is, $\bX_{ij}\bX_{jk}\leq \bX_{ik}$ means that  if keypoints
$t_i \in [m_i]$ and $t_k \in [m_k]$ (in images $i$ and $k$, respectively) are both matched to a keypoint $t_j$ in image $j$, then they are matched to each other. Therefore, any cycle $ijk$ with corresponding partial permutations $\bX_{ij}$, $\bX_{jk}$, $\bX_{ki}$  is consistent if and only if for any $t_i \in [m_i]$, $t_j \in [m_j]$ and $t_k \in [m_k]$: If two of the events $\bX_{ij}(t_i,t_j)=1$, $\bX_{jk}(t_j,t_k)=1$, $\bX_{ki}(t_k,t_i)=1$ hold true, then the third one holds true as well.

This equivalent reformulation of cycle consistency further clarifies the definition of $d_{ijk}$ in \eqref{eq:dijk}.
For fixed $\bX_{ij}, \bX_{jk}, \bX_{ki} \in \mathcal{P}^{l_1,l_2}$, the denominator of the fraction in \eqref{eq:dijk} can be viewed as the number of combinations of three keypoints $a$, $b$, $c$,
 such that at least two of the three events
 \begin{equation}
 \label{eq:3_cond}
 \bX_{ij}(a,b)=1, \ \bX_{jk}(b,c)=1, \text{ and }  \bX_{ki}(c,a)=1
\end{equation}
hold. Furthermore, the numerator of the fraction in \eqref{eq:dijk} can be viewed as the total number combinations of three keypoints $a$, $b$, $c$,
 such that all the three events in \eqref{eq:3_cond} hold. Thus, the fraction in \eqref{eq:dijk} indeed measures the level of cycle consistency, and consequently $d_{ijk}$ measures the cycle inconsistency.

We remark that an inequality of two full permutation matrices must be an equality. Therefore, for permutation synchronization the above definition of cycle consistency is equivalent with $\bX_{ij}\bX_{jk} = \bX_{ik}$ (or equivalently, $\bX_{jk}\bX_{ki} = \bX_{ji}$ or $\bX_{ki}\bX_{ij} = \bX_{kj}$ or
$\bX_{ki}\bX_{ij} \bX_{jk} = \bI$). That is, our definition of cycle consistency is a direct extension of the one in group synchronization.

\subsection{Cycle-verifiability Helps in Verifying Matches in Cycles}
We further interpret the cycle-verifiable condition and clarify its name. We claim that if $ijk$ is a good cycle (w.r.t.~$ij$) ensured by Definition \ref{def:cycle_verifiable} with $a \in I_i$ and
$c \in I_j$, then one can verify whether $a$ and $c$ correctly match (i.e., $h(a) = h(c)$) using $b \in I_k$.
Indeed, since $b$ matches $a$ and $k \in G_{ij}$, $h(a) = h(b)$. If $b$ and $c$ match then since $k \in G_{ij}$ $h(b)=h(c)$ and consequently $h(a) = h(c)$. Assume on the other hand that $b$ and $c$ do not match. If $b$ matches another point $c'$, then since $k \in G_{ij}$, $h(b)=h(c') \neq h(c)$. If $b$ does not match any point in $U_j$, then since $k \in G_{ij}$, $h(b) \not \in h(U_j)$ (otherwise there exists $c' \in U_j$ such that $h(c')=h(b)$ and since $k \in G_{ij}$ there has to be a match between $b$ and $c'$.). Since $h(c) \in h(U_j)$ and $h(b) \not \in h(U_j)$, $h(b) \neq h(c)$.

\section{Proof of Theorem \ref{thm:main}}
The proof establishes two lemmas, Lemmas \ref{lemma1} and \ref{lemma2}, and then uses them to conclude Theorem \ref{thm:main}. It is rather technical and not so easy to motivate. In order to provide more intuition, we added some clarifying figures.

{\bf Convention for figures:} In all of these figures, we designate by green lines good keypoint matches, by red lines bad keypoint matches and by dashed orange lines missing keypoint matches. All of these occur between keypoints of two different images.
On the other hand, matches between keypoints in an image and universal 3D keypoints are designated by black dotted lines (these correspond to our formal $h$ function). We further color the universal 3D keypoints (in $U_{ij}$), which represent elements of $U_{ij}^{\text{good}}$, by green. We also color the universal 3D keypoints, which represent elements of $U_{ij}^{\text{bad}}$, in red.
In Figure \ref{fig:Lemma1_1}, we slightly extend the latter convention and explain it in its caption.

{\bf Terminology Review:}
Recall that $n_{\Delta} = \tr(\bX_{ij} \bX_{jk} \bX_{ki})$,
$m$ is the number of all 3D keypoints, $m_{ij}$ is the number of 3D keypoints that correspond to the 2D keypoints of images $i$ or $j$ and among these, $m_{ij}^{\text{bad}}$ is the number of keypoints that match wrong keypoints in the other image, or match no keypoint if a ground-truth match exists.
We also denote the number of the rest of points by $m_{ij}^{\text{good}}$
(that is, $m_{ij}^{\text{good}} = m_{ij} -m_{ij}^{\text{bad}}$)
and recall that these keypoints match the ground-truth keypoints or, do not match any keypoint, if no ground-truth matches exist.

\subsection{Upper Bound for the Cycle Inconsistency of Good Edges}
This section includes the proof of the following lemma:
\begin{lemma}
\label{lemma1}
For any $ij \in E_g$, $d_{ijk} \le m (s_{ik}^* + s_{jk}^*)$.
\end{lemma}
We remark that in the case of group synchronization, in particular, PS, one can easily show that for any $ij \in E$, $|d_{ijk}-s_{ij}^*| \le s_{ik}^* + s_{jk}^*$ (see Lemma 1 of \cite{cemp}). Consequently for $ij \in E_g$,  $d_{ijk} \le s_{ik}^* + s_{jk}^*$. However, in PPS, without the full group structure with a bi-invariant metric, it is harder to prove the weaker bound of Lemma \ref{lemma1}. The proof below involves various discrete combinatorial arguments.

\begin{proof}
Assume first that $n_{\Delta} = 0$ and note that \eqref{eq:dijk} implies $d_{ijk}=1$. Since $d_{ijk} \not =0$ and $ij \in E_g$, $s_{ik}^*$ and $s_{jk}^*$ cannot be both zero (otherwise this and the fact that $ij \in E_g$ imply that $ijk$ is cycle-consistent and thus $d_{ijk}=0$). Without loss of generality, assume $s_{jk}^*>0$. We note that
$$s_{jk}^* = \frac{m_{jk}^{\text{bad}}}{m_{jk}} \ge \frac{1}{m},$$
which implies the desired bound:
$$d_{ijk} = 1 = m \cdot \frac{1}{m} \le m s_{jk}^* \le m(s_{ik}^* + s_{jk}^*).$$

Assume next that $n_{\Delta}>0$, or equivalently,
\begin{equation}
\label{eqn:lemma1.1}
n_{\Delta} \ge 1.
\end{equation}

The next arguments require additional definitions and observations.
We recall that any element of $I_i$ represents a 2D keypoint in image $i$.
This keypoint is associated with the index vector $(i,j)$, where $j=1,\ldots,m_i$ and we can thus view $I_i$ as the set of $m_i$ index vectors.
For cycle $ijk$, $(a,b,c) \in I_i \times I_j \times I_k$ is an $(i,jk)$ tuple if there is a match between $a$ and both $b$ and $c$ (the match can be either good or bad). If $(a,b,c)$ is an $(i,jk)$ tuple and there is no match between $b$ and $c$, then we refer to $(a,b,c)$ as a bad $(i,jk)$ tuple, otherwise, it is a good $(i,jk)$ tuple. For example, in Figure \ref{fig:Lemma1_1}, there are three $(i,jk)$ tuples: $(p_{i,2},p_{j,1},p_{k,1})$, $(p_{i,3},p_{j,3},p_{k,2})$ and $(p_{i,4},p_{j,2},p_{k,3})$. We note that $(p_{i,2},p_{j,1},p_{k,1})$ and $(p_{i,4},p_{j,2},p_{k,3})$ are bad $(i,jk)$ tuples and $(p_{i,3},p_{j,3},p_{k,2})$ is a good $(i,jk)$ tuple.
For cycle $ijk$, we denote by $A_{i,jk}$ the set of $(i,jk)$ tuples in $I_i \times I_j \times I_k$, by $A_{i,jk}^{\text{bad}}$ the set of bad $(i,jk)$ tuples and by $A_{i,jk}^{\text{good}}$ the set of good $(i,jk)$ tuples.

Recall that for a cycle $ijk$, $n_i = \nnz(\bX_{ki}\bX_{ij})$, $n_j$ and $n_k$ are analogously defined, and  $n_{\Delta} = \tr(\bX_{ij}\bX_{jk}\bX_{ki})$.
Also recall the notation $p_{i,j}$ for elements of $I_i$.
We note that $(a,b,c)$ is an $(i,jk)$ tuple if and only if (iff)
$a=p_{i,u}$, $b=p_{j,v}$, $c=p_{k,w}$ and
both $\bX_{ki}(w,u)=1$ and $\bX_{ij}(u,v)=1$ (so $p_{i,u}$ matches to both $p_{j,v}$ and $p_{k,w}$). We further note that the latter two requirements are equivalent with $\bX_{ki}\bX_{ij}(w,v)=1$. Indeed, since $\bX_{ki}\bX_{ij}(w,v)= \sum_{u' \in [m_i]}\bX_{ki}(w,u') \bX_{ij}(u',v)$
and $\bX_{ki}\bX_{ij}$ is a partial permutation, then  $\bX_{ki}\bX_{ij}(w,v)=1$ iff $\bX_{ki}(w,u)=\bX_{ij}(u,v)=1$ for some $u \in [m_i]$.
Therefore
$$|A_{i,jk}| = n_i.$$
By the same way we note that
$$|A_{j,ki}| = n_j \ \text{ and } \ |A_{k,ij}| = n_k.$$

Similarly, note that  $(a,b,c)$ is a good $(i,jk)$ tuple iff $a=p_{i,u}$, $b=p_{j,v}$, $c=p_{k,w}$, $\bX_{ki}(w,u)=1$, $\bX_{ij}(u,v)=1$ and $\bX_{jk}(v,w)=1$.
The latter three requirements are equivalent with
$\bX_{ij}\bX_{jk}\bX_{ki}(u,u)=1$.
Indeed, the following equation
\begin{multline*}
\bX_{ij}\bX_{jk}\bX_{ki}(u,u)=\\
\sum_{v' \in [m_j],w' \in [m_{k}]} \bX_{ij}(u,v')\bX_{jk}(v',w')\bX_{ki}(w',u)
\end{multline*}
and the fact that $\bX_{ij}\bX_{jk}\bX_{ki}$ is a partial permutation imply this equivalence. Therefore,
$$|A_{i,jk}^{\text{good}}| = n_{\Delta} \ \text{  and } \ |A_{i,jk}^{\text{bad}}| = n_i - n_{\Delta}.$$ Similarly, we conclude that
$$
|A_{j,ki}^{\text{bad}}| = n_j - n_{\Delta}
 \ \text{  and } \ |A_{k,ij}^{\text{bad}}| = n_k - n_{\Delta}.$$
Using these observations and \eqref{eq:dijk}, one can rewrite $d_{ijk}$ as follows
\begin{equation}
\label{eqn:lemma1.2}
d_{ijk} = \frac{|A_{i,jk}^{\text{bad}}| + |A_{j,ki}^{\text{bad}}| + |A_{k,ij}^{\text{bad}}|}{|A_{i,jk}^{\text{bad}}| + |A_{j,ki}^{\text{bad}}| + |A_{k,ij}^{\text{bad}}| + 3n_{\Delta}}.
\end{equation}

Let us assume that $(a,b,c) \in A_{i,jk}^{\text{bad}}$ and show that
\begin{equation}
    \label{eq:hc_property}
h(c)\in U_{ik}^{\text{bad}} \cup U_{jk}^{\text{bad}}.
\end{equation}
The assumption $ij \in E_g$ implies that there is a good match between $a$ and $b$.

We claim that if there is also a good match between $a$ and $c$, then $h(b) = h(a) = h(c) \in U_{ij} \cap U_{ik} \cap U_{jk}$.
Indeed, assume $a = p_{i,u}$, $b = p_{j,v}$, $c = p_{k,w}$ and $h(a) = l$, i.e., $\bP_i^*(u,l) = 1$. Because there exists a good match between $a$ and $b$, $\bX_{ij}^*(u,v) = \bX_{ij}(u,v) = 1$. Since $\bX_{ij}^*(u,v) = \bP_i^* \bP_j^{*T}(u,v) = \sum_{d \in [m]}\bP_i^*(u,d) \bP_j^*(v,d)$ and $\bP_j^*$ is a partial permutation, $\bP_j^*(v,l) = 1$ and thus $h(b)=l$. Similarly, since there exists a good match between $a$ and $c$, $\bP_k^*(w,l) = 1$ and $h(c)=l$. Therefore $h(a) = h(b) = h(c) = l \in U_{ij} \cap U_{ik} \cap U_{jk}$.

\begin{figure}
    \centering
    \includegraphics[width = 8cm]{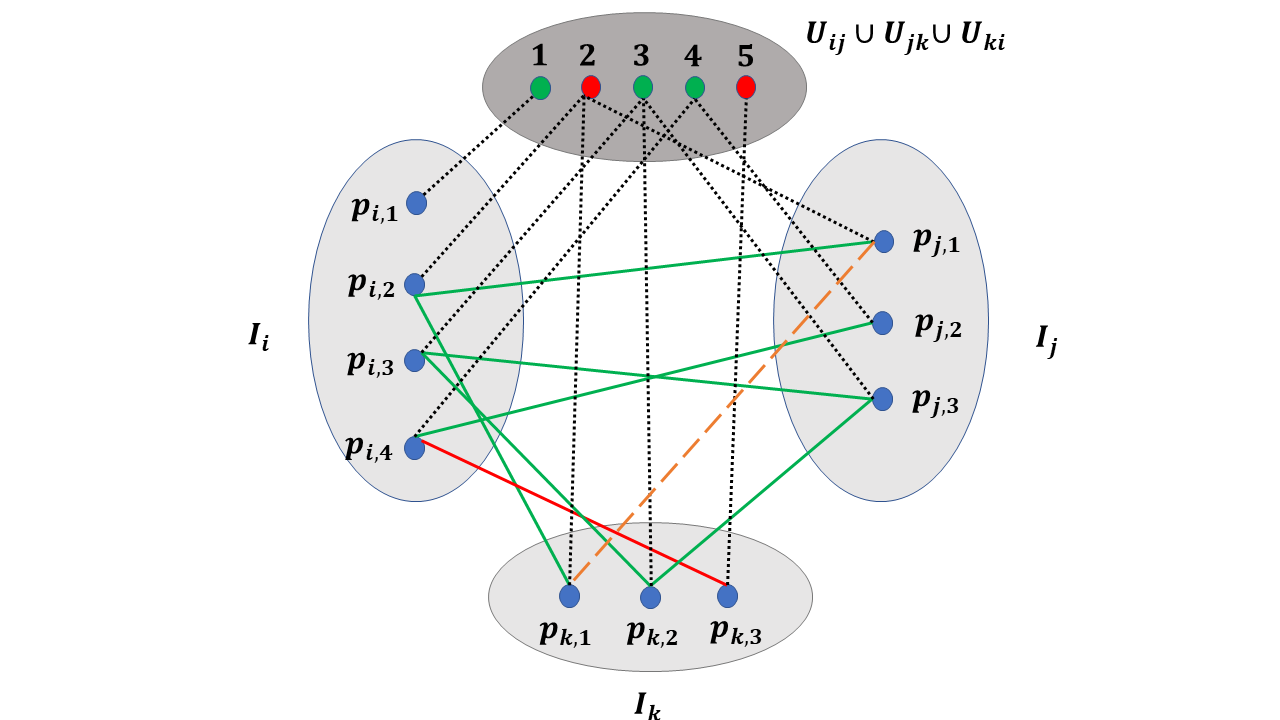}
    \caption{A demonstration for clarifying the definition of $(i,jk)$ tuples, good $(i,jk)$ tuples, bad $(i,jk)$ tuples as well as the function $f$. The actual use of the figure is clarified when it is referred to.
    Unlike other figures (that focus on points in $U_{ij}$ and not in $U_{ij} \cup U_{jk} \cup U_{ki}$), the red dots correspond to keypoints in $U_{ij}^{\text{bad}} \cup U_{jk}^{\text{bad}} \cup U_{ik}^{\text{bad}}$ and the green dots correspond to the rest of keypoints in $U_{ij} \cup U_{jk} \cup U_{ki}$. We note that $U_{ik}^{\text{bad}} = \{ 5 \}$, $U_{jk}^{\text{bad}} = \{ 2 \}$ and $U_{ij}^{\text{bad}} = \varnothing$. }
    \label{fig:Lemma1_1}
\end{figure}

Since $(a,b,c) \in A_{i,jk}^{\text{bad}}$, there is no match between $b$ and $c$ and thus $h(c) \in U_{jk}^{\text{bad}}$, which implies \eqref{eq:hc_property}.

If on the other hand, there is a bad match between $a$ and $c$, then $h(c) \in U_{ik}^{\text{bad}}$, which also implies \eqref{eq:hc_property}.

In view of \eqref{eq:hc_property}, the function $f(a,b,c) = h(c)$ maps $A_{i,jk}^{\text{bad}}$ to $U_{ik}^{\text{bad}} \cup U_{jk}^{\text{bad}}$.
We note that this function is injective. Indeed, since $\bX_{ik}$ and $\bX_{jk}$ are partial permutations, for any $h(c) \in U_{ik}^{\text{bad}} \cup U_{jk}^{\text{bad}}$, where $c \in I_k$, if there are $a$ and $b$ in $I_i$ and $I_j$, respectively, such that there are matches between them and $c$ and $f(a,b,c) =h(c)$, then $a$ and $b$ are unique. Figure \ref{fig:Lemma1_1} demonstrates $f$ and its proven injectivity in a special case. In this case, $A_{i,jk}^{\text{bad}} = \{ (p_{i,4},p_{j,2},p_{k,3}), (p_{i,2},p_{j,1},p_{k,1})\} $, $f (p_{i,4},p_{j,2},p_{k,3}) = h(p_{k,3}) = 5  \in U_{ik}^{\text{bad}}$ and $f (p_{i,2},p_{j,1},p_{k,1}) = h(p_{k,1}) = 2 \in U_{jk}^{\text{bad}}$.

By the cardinality property of an injective map,
\begin{equation}
\nonumber
|A_{i,jk}^{\text{bad}}| \le |U_{ik}^{\text{bad}} \cup U_{jk}^{\text{bad}}| \le |U_{ik}^{\text{bad}}| + |U_{jk}^{\text{bad}}| = m_{ik}^{\text{bad}} + m_{jk}^{\text{bad}}.
\end{equation}
Similarly, the same bound holds for $|A_{j,ki}^{\text{bad}}|$ and $|A_{k,ij}^{\text{bad}}|$
and consequently,
\begin{equation}
\label{eqn:lemma1.4}
|A_{i,jk}^{\text{bad}}| +  |A_{j,ki}^{\text{bad}}| + |A_{k,ij}^{\text{bad}}| \le  3(m_{ik}^{\text{bad}} + m_{jk}^{\text{bad}}).
\end{equation}
Therefore the combination of
\eqref{eqn:lemma1.1} and \eqref{eqn:lemma1.4} with the definition of $s_{ik}^*$ as well as $s_{jk}^*$ (i.e., noting that
$s_{ik}^* = {m_{ik}^{\text{bad}}}/{m_{ik}}= 1-{m_{ik}^{\text{good}}}/{m_{ik}}$ and
$s_{jk}^* = {m_{jk}^{\text{bad}}}/{m_{jk}}=1-{m_{jk}^{\text{good}}}/{m_{jk}}$)
yields
\begin{align*}
    d_{ijk} &= \frac{|A_{i,jk}^{\text{bad}}| + |A_{j,ki}^{\text{bad}}| + |A_{k,ij}^{\text{bad}}|}{|A_{i,jk}^{\text{bad}}| + |A_{j,ki}^{\text{bad}}| + |A_{k,ij}^{\text{bad}}| + 3n_{\Delta}}\\
    & \le \frac{3(m_{ik}^{\text{bad}} + m_{jk}^{\text{bad}})}{3(m_{ik}^{\text{bad}} + m_{jk}^{\text{bad}})+3} \\
    & \le \frac{3 m_{ik}^{\text{bad}}}{3m_{ik}^{\text{bad}} + 3} + \frac{3 m_{jk}^{\text{bad}}}{3m_{jk}^{\text{bad}} + 3}\\
    & = \frac{3s_{ik}^*}{3s_{ik}^* + \frac{3}{m_{ik}}} + \frac{3s_{jk}^*}{3s_{jk}^* + \frac{3}{m_{jk}} }\\
    & \le m_{ik}s_{ik}^* +
    m_{jk} s_{jk}^* \\
    & \le m (s_{ik}^* + s_{jk}^*).
\end{align*}
\end{proof}

\subsection{Lower Bound for the Averaged Cycle Inconsistency Among Good Cycles}
This section includes the proof of the following lemma:

\begin{lemma}
\label{lemma2}
If $G = (V,E)$ is $p_v$-cycle verifiable, then
\begin{equation}
\label{eq:lemma2}
\frac{1}{3}p_v s_{ij}^* \le
\frac{1}{|G_{ij}|}\sum_{k \in G_{ij}} d_{ijk} \quad \forall ij \in E.
\end{equation}
\end{lemma}
\begin{proof}
We assume several cases.

\noindent {\bf Case I:} \boldsymbol{$ij \in E_g$}.
The left hand side of \eqref{eq:lemma2} is zero and its right hand side (RHS) is also zero since for any $ij \in E_g$ and $k \in G_{ij}$, $d_{ijk}=0$.

\noindent {\bf Case II:} \boldsymbol{$ij \in E_b$} {\bf and} $\boldsymbol{k \in G_{ij}}$.
Denote
$$N_{ijk}^{\text{bad}} = |A_{i,jk}^{\text{bad}}| + |A_{j,ki}^{\text{bad}}| +  |A_{k,ij}^{\text{bad}}|,$$ and note that in view of \eqref{eqn:lemma1.2}
\begin{equation}
\label{eqn:lemma2.4}
\sum_{k \in G_{ij}} d_{ijk} = \sum_{k \in G_{ij}} \frac{N_{ijk}^{\text{bad}}}{N_{ijk}^{\text{bad}} + 3n_{\Delta}}.
\end{equation}
We thus need to lower bound the RHS of \eqref{eqn:lemma2.4} in order to conclude \eqref{eq:lemma2}.

We first derive the bound $n_{\Delta} \le m_{ij}^{\text{good}}$. Figure \ref{fig:Lemma2_1} demonstrates the definitions below and the desired bound in a very special case.
Let $D_{ijk}$ denote the set of indices of diagonal entries of $\bX_{ij} \bX_{jk} \bX_{ki}$ that equal 1.
Note that $|D_{ijk}| = n_{\Delta}$ due to the fact that $n_{\Delta} = \tr(\bX_{ij} \bX_{jk} \bX_{ki})$.
Also, since $\bX_{ij} \bX_{jk} \bX_{ki}$ is of size $m_i \times m_i$, $D_{ijk} \subseteq [m_i]$.
We can thus assign for $d \in D_{ijk}$, $p_{i,d} = a \in I_i$. Because $\bX_{ij} \bX_{jk} \bX_{ki}(d,d)=1$, there exists $1 \le u \le m_j$ and $1 \le v \le m_k$ such that $\bX_{ij}(d,u) = \bX_{jk}(u,v) = \bX_{ki}(v,d) = 1$. Therefore, we note that for $b := p_{j,u} \in I_j$ and $c := p_{k,v} \in I_k$, there exist matches between $a$ and $b$, $b$ and $c$, as well as $c$ and $a$. Since $k \in G_{ij}$, $jk \in E_g$ and $ki \in E_g$ and thus $h(b) = h(c) = h(a)$ (see the same argument in the paragraph below \eqref{eq:hc_property}, where it is enough to just assume that either $jk \in E_g$ or $ki \in E_g$); we denote the latter common value by $l$. Therefore $\bP_i^*(a,l) = \bP_j^*(b,l) = \bX_{ij}(a,b) = 1$. By definition of $U_{ij}^{\text{good}}$, $l \in U_{ij}^{\text{good}}$.
Let $f_{ijk}$ be a function from $D_{ijk}$ to $U_{ij}^{\text{good}}$ such that $f_{ijk}(d) = l$.
We note that it is injective since for any $d \not = d' \in D_{ijk}$, $p_{i,d} \not = p_{i,d'}$, therefore $h(p_{i,d}) \not = h(p_{i,d'})$. By the cardinality property of an injective map, $n_{\Delta} \le |U_{ij}^{\text{good}}| = m_{ij}^{\text{good}}$.

\begin{figure}
    \centering
    \includegraphics[width=8cm]{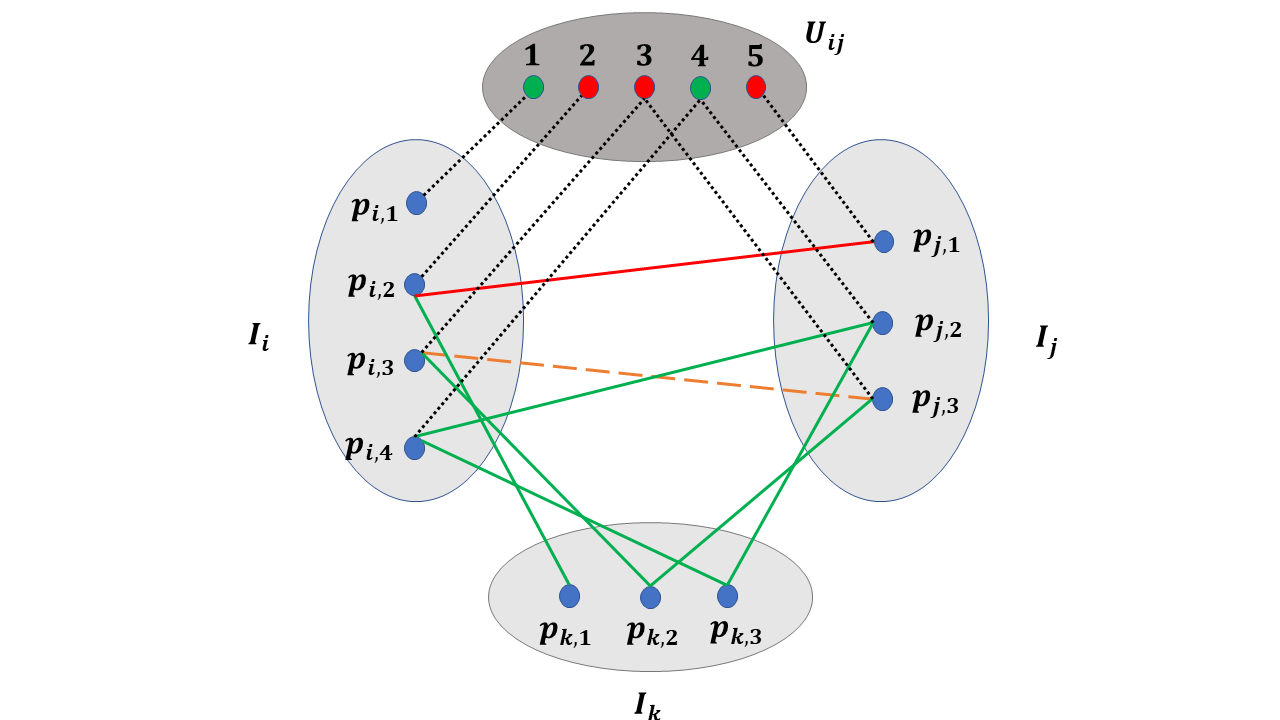}
    \caption{ An illustration of $n_{\Delta}$, $D_{ijk}$, $f_{ijk}$ and $m_{ij}^{\text{good}}$.
    Note that $n_{\Delta}$ is equal to the number of green triangles with vertices in $I_i, I_j, I_k$. Since there is only one such triangle, $n_{\Delta}=1$. This triangle (with keypoints $p_{i,4},p_{j,2},p_{k,3}$) is associated with the keypoint in $I_i$ with index 4 and thus $D_{ijk}=\{ 4 \}$. Since it is also associated with the universal keypoint with index 4, the function $f_{ijk}$ maps $4$ to $4 \in U_{ij}^{\text{good}}$. At last, note that  $m_{ij}^{\text{good}} = |U_{ij}^{\text{good}}| = 2$. }
    \label{fig:Lemma2_1}
\end{figure}


\begin{figure}
    \centering
    \includegraphics[width = 8cm]{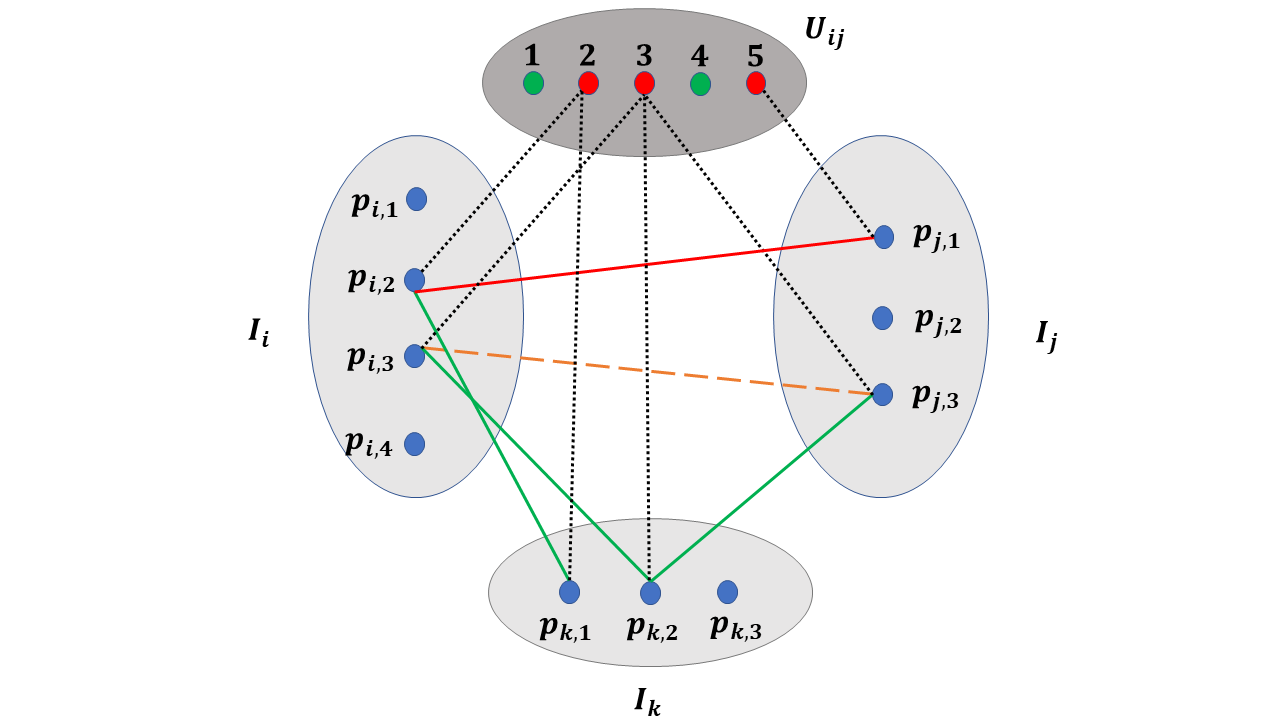}
    \caption{An illustration of $g_i$, $g_{ij,k}$ and the injectivity of both functions in a very special case. In this example, the only element of $A_{i,jk}^{\text{bad}}$ is $(p_{i,2},p_{j,1},p_{k,1})$. Since we defined $g_i(a,b,c) = h(a)$, we obtain that $g_i(p_{i,2},p_{j,1},p_{k,1}) = h(p_{i,2}) = 2 \in U_{ij}^{\text{bad}}$.
    Note that $U_{ij}^{\text{bad}} \cap h(I_k) = \{ 2,3 \}$. Recall that the function $g_{ij,k}$ maps $x \in U_{ij}^{\text{bad}} \cap h(I_k)$ to the bad $(j,ki)$ tuple or bad $(k,ij)$ tuple that involves $c \in I_k$ with $h(c) = x$. In this example, the function $g_{ij,k}$ maps $2 \in U_{ij}$ to the bad $(i,jk)$ tuple $(p_{i,2},p_{j,1},p_{k,1})$. It maps $3 \in U_{ij}$ to the bad $(k,ij)$ tuple $(p_{k,2},p_{i,3},p_{j,3})$.}
    \label{fig:Lemma2_2}
\end{figure}

Next, We prove an upper bound of $N_{ijk}^{\text{bad}}$.
We assume without loss of generality that  $(a,b,c) \in A_{i,jk}^{\text{bad}}$.
Since $ik,jk \in E_g$, the matches from $a$ to $c$ and from $b$ to $c$ are correct. Therefore, the match from $a$ to $b$ is wrong and $h(a) \in U_{ij}^{\text{bad}}$. Denote by $g_i: A_{i,jk}^{\text{bad}} \rightarrow U_{ij}^{\text{bad}}$ the function which maps  $(a,b,c) \in A_{i,jk}^{\text{bad}}$ to $h(a) \in U_{ij}^{\text{bad}}$. Figure \ref{fig:Lemma2_2} illustrates $g_i$ in a special case. This function is injective since
for any $x \in U_{ij}^{\text{bad}}$, $g_i^{-1}(x)$ contains at most one element $(a,b,c)\in A_{i,jk}^{\text{bad}}$.
Indeed, if $g_i(a,b,c) = x$, then there must exist $a \in I_i$ such that $h(a) = x$, $c$ in $I_k$ such that there is a match between $a$ and $c$, and $b$ in $I_j$ such that there is a match between $a$ and $b$ (and no match between $b$ and $c$).
Note that there is a match between at most one keypoint in $I_k$ and $a$ and thus there is at most one such  $c$.
Similarly, there is at most one such $b$.
Since there is at most one keypoint in $I_i$ which corresponds to the 3D keypoint $x$, there is at most one such $a$. 
The  injectivity of $g_i$ implies $|A_{i,jk}^{\text{bad}}| \le |U_{ij}^{\text{bad}}| = m_{ij}^{\text{bad}}$. Similarly, $|A_{j,ki}^{\text{bad}}| \le  m_{ij}^{\text{bad}}$ and $|A_{k,ij}^{\text{bad}}| \le m_{ij}^{\text{bad}}$.  Thus, for any $k \in G_{ij}$
\begin{equation}
    \label{eqn:lemma2.1}
    0 \le N_{ijk}^{\text{bad}} \le 3 m_{ij}^{\text{bad}}.
\end{equation}

Next, we establish a lower bound of $N_{ijk}^{\text{bad}}$.
For this purpose, we construct an injective map $g_{ij,k}$ from $U_{ij}^{\text{bad}} \cap h(I_k)$ to $A_{i,jk}^{\text{bad}} \cup A_{j,ki}^{\text{bad}} \cup A_{k,ij}^{\text{bad}}$. It will allow us to lower bound $N_{ijk}^{\text{bad}} = |A_{i,jk}^{\text{bad}}| + |A_{j,ki}^{\text{bad}}| + |A_{k,ij}^{\text{bad}}|$ by the cardinality of $U_{ij}^{\text{bad}} \cap h(I_k)$. Note that $U_{ij}^{\text{bad}} \subseteq U_{ij} = h(I_i) \cup h(I_j)$.
Therefore any element of $U_{ij}^{\text{bad}} \cap h(I_k)$ is either in $h(I_i)$ or $h(I_j) \backslash h(I_i)$.

In the case where $x \in U_{ij}^{\text{bad}} \cap h(I_k)$ and $x \in h(I_i)$, we will show that there exist either $(a,b,c) \in A_{i,jk}^{\text{bad}}$ or $(c,a,b) \in A_{k,ij}^{\text{bad}}$ such that $h(c) = x$.
In the case where $x \in U_{ij}^{\text{bad}} \cap h(I_k)$ and $x \in h(I_j) \backslash h(I_i)$, then one can similarly show that there exists either $(b,c,a) \in A_{j,ki}^{\text{bad}}$ or $(c,a,b) \in A_{k,ij}^{\text{bad}}$ such that $h(c) = x$. These arguments induce a map $g_{ij,k}$ from $U_{ij}^{\text{bad}} \cap h(I_k)$ to $A_{i,jk}^{\text{bad}} \cup A_{j,ki}^{\text{bad}} \cup A_{k,ij}^{\text{bad}}$ which maps $x$ to its corresponding bad tuple. Since $h(c) = x$, $g_{ij,k}$ is injective. Figure \ref{fig:Lemma2_2} illustrates $g_{ij,k}$ in a special case.

We thus assume that $x \in U_{ij}^{\text{bad}} \cap h(I_k)$ and $x \in h(I_i)$. Note that the latter requirement implies the existence of $a \in I_i$ such that $h(a) = x$.
Since $x \in h(I_k)$, there exists $c \in I_k$ such that $h(c) = x$ and since $ik \in E_g$ there is a good match between $a$ and $c$.
Note that there cannot be a good match between $a$ and any $b \in I_j$, otherwise  $x \not \in U_{ij}^{\text{bad}}$. Therefore, there are two cases to consider.
In the first case there exists $b \in I_j$ such that there is a wrong match between $a$ and $b$. This implies that $h(b) \neq h(a)$ and since we showed above that $h(a) = h(c)$, we conclude that $h(b) \neq h(c)$. The latter observation and the fact that $jk \in E_g$ imply that there is no match between $b$ and $c$ and thus $(a,b,c)$ is a bad $(i,jk)$ tuple, that is, $(a,b,c) \in A_{i,jk}^{\text{bad}}$.
In the second case, there exists $b \in I_j$ such that $h(b) = h(a)$, but there is no match between $b$ and $a$
(the previous case considered the scenario where there exists $b \in I_j$ such that $a$ and $b$ match; furthermore, if $h(a) \neq h(b)$ for all $b \in I_j$, then $x \in U_{ij}^{\text{good}}$).
Since $h(a) = h(c) = h(b)$ and $jk \in E_g$, there is a match between $b$ and $c$. Therefore, $(c,a,b)$ is a bad $(k,ij)$ tuple, that is, $(c,a,b) \in A_{k,ij}^{\text{bad}}$. Following the above ideas, this concludes the injectivity of $g_{ij,k}$.
This injectivity implies
\begin{align}
\label{eqn:lemma2.2}
    \sum_{x \in U_{ij}^{\text{bad}}} 1_{\{x \in h(I_k)\}} & = |U_{ij}^{\text{bad}} \cap h(I_k)|  \le |A_{i,jk}^{\text{bad}} \cup A_{j,ki}^{\text{bad}} \cup A_{k,ij}^{\text{bad}}| \nonumber \\ & \le |A_{i,jk}^{\text{bad}}| + |A_{j,ki}^{\text{bad}}| +  |A_{k,ij}^{\text{bad}}| = N^{\text{bad}}_{ijk}.
\end{align}

In order to apply \eqref{eqn:lemma2.2} we lower bound a certain sum of $1_{\{x \in h(I_k)\}}$. Our argument assumes that $x \in U_{ij}^{\text{bad}}$. Since $x \in U_{ij}$, we conclude  WLOG that $x \in h(I_i)$. Therefore there exists $a \in I_i$ such that $h(a) = x$. By the $p_v$-cycle verifiability condition, $a$ is verifiable w.r.t. $ij$ in at least $p_v|G_{ij}|$ good cycles. For any such cycle $ijk$ that $a$ is verifiable in, let $b \in I_k$ match $a$ (for convenience, we demonstrate $x$, $a$ and $b$ in Figure \ref{fig:Lemma2_3}). Since $ik \in E_g$, the match between $a$ and $b$ is a good match and thus $x = h(a) = h(b)$. Since $b \in I_k$, $h(b) \in h(I_k)$ and thus $x \in h(I_k)$. That is, we have proved that if $k \in G_{ij}$ and $a$ is verifiable in $ijk$, then $x \in h(I_k)$. We have at least $p_v |G_{ij}|$ such $k$'s and thus
\begin{equation}
\label{eq:bound_with_pv_cycle}
    \sum_{k\in G_{ij}} 1_{\{x \in h(I_k)\}} \ge p_v |G_{ij}| \ \text{ for any } x \in U_{ij}.
\end{equation}

\begin{figure}
    \centering
    \includegraphics[width=8cm]{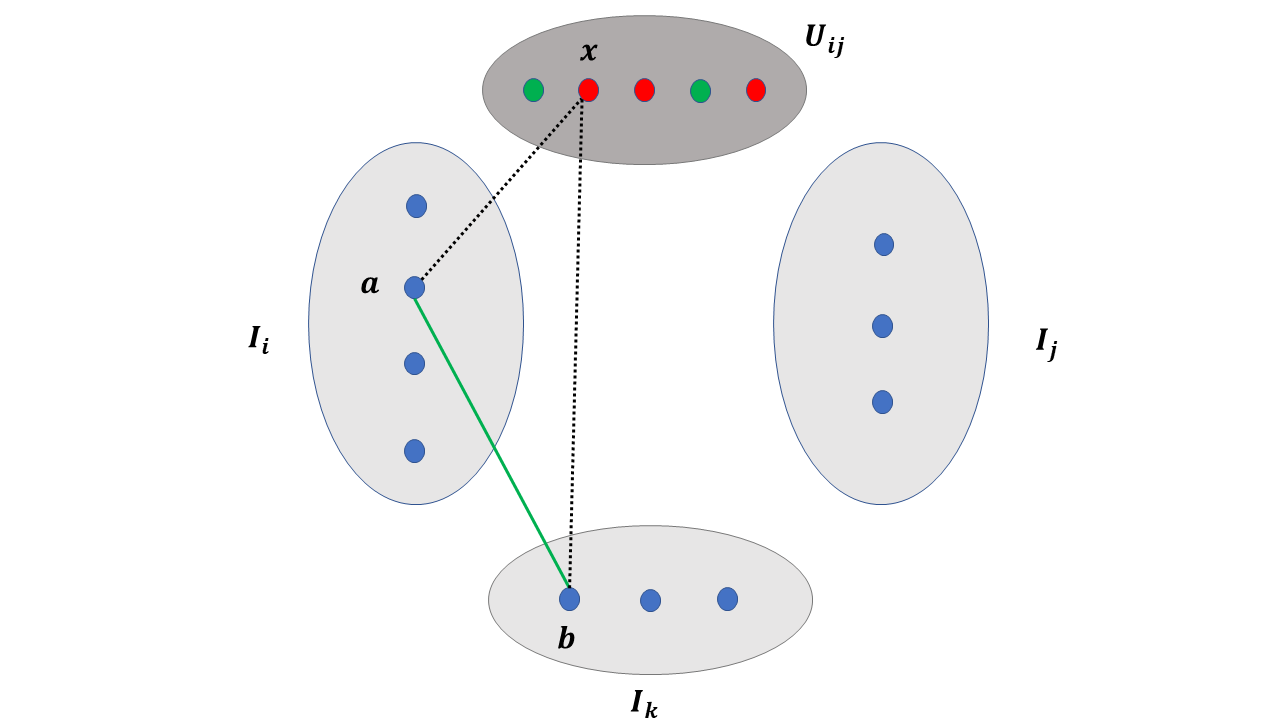}
    \caption{Visual demonstration of keypoints involved in the argument for bounding $\sum_{k\in G_{ij}} 1_{\{x \in h(I_k)\}}$.
    }
    \label{fig:Lemma2_3}
\end{figure}

We combine the above two inequalities as follows. Summing both sides of \eqref{eqn:lemma2.2} over $k \in G_{ij}$,  exchanging the order of summation and applying \eqref{eq:bound_with_pv_cycle}  result in
\begin{equation}
\label{eqn:lemma2.3}
p_v |G_{ij}| m_{ij}^{\text{bad}} = p_v |G_{ij}| |U_{ij}^{\text{bad}}| \le \sum_{k\in G_{ij}} N_{ijk}^{\text{bad}}.
\end{equation}

Using the above bound we will bound from below
$$\sum_{k \in G_{ij}} \frac{N_{ijk}^{\text{bad}}}{N_{ijk}^{\text{bad}} + 3m_{ij}^{\text{good}}}$$
and we will then use \eqref{eqn:lemma2.4} to conclude the desired inequality.
We denote
\begin{equation}
\label{eq:F_at_3mijbad}
    F(x) = \frac{x}{x+\gamma} \ \text{ where } \gamma = 3m_{ij}^{\text{good}}.
\end{equation}
Note that $F(0) = 0$,
$$F(3m_{ij}^{\text{bad}}) = \frac{3m_{ij}^{\text{bad}}}{3m_{ij}^{\text{bad}} + 3m_{ij}^{\text{good}}} = s_{ij}^*$$
and $F(x)$ is concave. Applying the definition of $F$, Jensen's inequality, \eqref{eqn:lemma2.3} and  \eqref{eq:F_at_3mijbad} yield
\begin{align*}
\sum_{k \in G_{ij}} \frac{N_{ijk}^{\text{bad}}}{N_{ijk}^{\text{bad}} + 3m_{ij}^{\text{good}}}
&= \sum_{k \in G_{ij}} F(N_{ijk}^{\text{bad}})\\
& \ge \sum_{k \in G_{ij}}((1-\frac{N_{ijk}^{\text{bad}}}{3m_{ij}^{\text{bad}}}) F(0) + \frac{N_{ijk}^{\text{bad}}}{3m_{ij}^{\text{bad}}} F(3m_{ij}^{\text{bad}}))\\
&= \sum_{k \in G_{ij}} \frac{N_{ijk}^{\text{bad}}}{3m_{ij}^{\text{bad}}} F(3m_{ij}^{\text{bad}})\\
&\ge \frac{1}{3} p_v |G_{ij}| F(3m_{ij}^{\text{bad}}) = \frac{1}{3} p_v |G_{ij}| s_{ij}^*.
\end{align*}
The combination of this inequality with \eqref{eqn:lemma2.4} concludes the proof of the lemma.
\end{proof}

\subsection{Conclusion of Theorem \ref{thm:main}}
We prove the main theorem by induction, using Lemmas \ref{lemma1} and \ref{lemma2}. For $t=0$, the definition of $s_{ij}^{(0)}$, Lemma \ref{lemma2} and the definition of $\lambda$ imply that for all $ij \in E$: $$s_{ij}^{(0)} = \frac{\sum_{k \in N_{ij}} d_{ijk}}{|N_{ij}|} \ge \frac{\sum_{k \in G_{ij}} d_{ijk}}{|N_{ij}|} \ge \frac{p_v}{3} \frac{|G_{ij}|}{|N_{ij}|}s_{ij}^* \ge \frac{p_v}{3} (1- \lambda)s_{ij}^*.$$
We further note by using again the above definitions and
the fact that for all $ij \in E$ $0 \le d_{ijk} \le 1$ that for $ij \in E_g$, $$s_{ij}^{(0)} = \frac{\sum_{k \in N_{ij}} d_{ijk}}{|N_{ij}|} = \frac{\sum_{k \in B_{ij}} d_{ijk}}{|N_{ij}|} \le \frac{\sum_{k \in B_{ij}} 1}{|N_{ij}|} \le \lambda \le \frac{1}{2\beta_0}.$$ Therefore, the theorem is proved when $t=0$.

Next, we assume that the theorem holds for iterations $0,1,\cdots,t$ and show that it also holds for iteration $t+1$. Applying the definition of $s_{ij}^{(t+1)}$, the positivity of the terms in the sum, the induction assumption  $\frac{1}{2\beta_t} \ge \max_{ij\in E_g} s_{ij}^{(t)}$, Lemma \ref{lemma2} and the definition of $\lambda$, we obtain for any $ij \in E_b$
\begin{align}
\label{eqn:mainpf1}
s_{ij}^{(t+1)} & = \frac{\sum_{k \in N_{ij}} e^{- \beta_t (s_{ik}^{(t)} + s_{jk}^{(t)})} d_{ijk}}{\sum_{k \in N_{ij}} e^{- \beta_t (s_{ik}^{(t)} + s_{jk}^{(t)})}} \nonumber \\
& \ge \frac{\sum_{k \in G_{ij}} e^{- \beta_t (s_{ik}^{(t)} + s_{jk}^{(t)})} d_{ijk}}{\sum_{k \in N_{ij}} e^{- \beta_t (s_{ik}^{(t)} + s_{jk}^{(t)})}} \nonumber \\
& \ge \frac{\sum_{k \in G_{ij}}e^{-1} d_{ijk}}{|N_{ij}|}  \\
& \ge \frac{p_v}{3e} \frac{|G_{ij}|}{|N_{ij}|}s_{ij}^* \nonumber\\
& \ge \frac{(1-\lambda)p_v}{3e}s_{ij}^* \nonumber.
\end{align}
Note that $xe^{-\alpha x} \le \frac{1}{\alpha e}$ for any $\alpha>0$ and $x \ge 0$. In particular, for $\alpha = \beta_t (1-\lambda)p_v/3e$ and $x = s_{ik}^* + s_{jk}^*$,
\begin{equation}
\label{eq:bound_1overea}
e^{-\beta_t (s_{ik}^* + s_{jk}^*) (1-\lambda)p_v/3e} (s_{ik}^* + s_{jk}^*) \le \frac{3}{\beta_t (1-\lambda)p_v}.
\end{equation}
Applying the definition of $s_{ij}^{(t+1)}$, the fact that $d_{ijk} = 0$ for any $ij \in E_g$ and $k \in G_{ij}$,
Lemma \ref{lemma1}, the induction assumption $s_{ij}^{(t)} \ge \frac{(1-\lambda)p_v}{3e}s_{ij}^*$ for $ij \in E_b$ (for the numerator) and the positivity of the relevant terms (for the denominator), the induction assumption $s_{ij}^{(t)} \leq {1}/{(2\beta_t)}$ for all $ij\in E_g$, \eqref{eq:bound_1overea} and the definition of $\lambda$,  we obtain
for all $ij \in E_g$
\begin{align*}
    s_{ij}^{(t+1)} & = \frac{\sum_{k \in N_{ij}} e^{- \beta_t (s_{ik}^{(t)} + s_{jk}^{(t)})} d_{ijk}}{\sum_{k \in N_{ij}} e^{- \beta_t (s_{ik}^{(t)} + s_{jk}^{(t)})}} \\
    & = \frac{\sum_{k \in B_{ij}} e^{- \beta_t (s_{ik}^{(t)} + s_{jk}^{(t)})} d_{ijk}}{\sum_{k \in N_{ij}} e^{- \beta_t (s_{ik}^{(t)} + s_{jk}^{(t)})}} \\
    &\le \frac{\sum_{k \in B_{ij}} e^{- \beta_t (s_{ik}^{(t)} + s_{jk}^{(t)})} m(s_{ik}^*+s_{jk}^*)}{\sum_{k \in N_{ij}} e^{- \beta_t (s_{ik}^{(t)} + s_{jk}^{(t)})}} \\
    & \le \frac{\sum_{k \in B_{ij}} e^{- \beta_t (s_{ik}^* + s_{jk}^*)\frac{(1-\lambda) \, p_v}{3e}} m(s_{ik}^*+s_{jk}^*)}{\sum_{k \in G_{ij}} e^{- \beta_t (s_{ik}^{(t)} + s_{jk}^{(t)})}}\\
    & \le \frac{m \sum_{k \in B_{ij}} e^{- \beta_t (s_{ik}^* + s_{jk}^*)\frac{(1-\lambda) \, p_v}{3e}} (s_{ik}^*+s_{jk}^*)}{|G_{ij}|e^{-1}}\\
    & \le \frac{m \sum_{k \in B_{ij}} \frac{3}{\beta_t (1-\lambda)p_v}}{|G_{ij}|e^{-1}} \\
    & = \frac{3m|B_{ij}|}{|G_{ij}|e^{-1}} \cdot \frac{1}{\beta_t (1-\lambda)p_v}\\
    & \le \frac{6em\lambda}{(1-\lambda)^2p_v} \cdot \frac{1}{2\beta_t}.
\end{align*}
We note that the assumption $\lambda < 1+\frac{3em}{p_v}-\sqrt{\frac{3em}{p_v}(2+\frac{3em}{p_v})}$ is equivalent with $\frac{6em\lambda}{(1-\lambda)^2p_v} < 1$. Therefore by taking $\beta_{t+1} = r\beta_t$ with $1 < r < \frac{(1-\lambda)^2p_v}{6em\lambda}$, we guarantee that for any $ij \in E_g$, $s_{ij}^{(t+1) } \le \frac{1}{2\beta_{t+1}}$, that is, $\max_{ij \in E_g} s_{ij}^{(t+1) } \le \frac{1}{2\beta_{t+1}} = \frac{1}{2\beta_0 r^t}$. This implication and \eqref{eqn:mainpf1} conclude the proof of the theorem.

\section{Discussion of a Possible Theoretical Extension}

Although our current analysis assumes no noise on the set of good edges, one can relax this assumption. Indeed, one can assume sufficiently small noise on good edges so that for all cycles $ijk$ and a sufficiently small positive constant $\delta$: $|d_{ijk}' - d_{ijk}| < \delta$, where $d_{ijk}$ and $d_{ijk}'$ are respectively the cycle inconsistencies with and without noise on good edges. Using a basic perturbation analysis, similarly as in the proof of Theorem \ref{thm:main}, with a carefully chosen set of the reweighting parameters $\beta_t$, one can prove approximate separation of good and bad edges. In particular, the maximum value of the estimated $s_{ij}$ on good edges is proportional to $\delta$.
Removing the bad edges (with estimated $s_{ij}$ larger than this threshold), one can then approximately solve the PPS problem with a subsequent spectral solver.
An approximate recovery theorem for the absolute partial permutations using the filtered edges can be established using spectral graph theory.

\end{document}

%% file: macro_file.tex
\usepackage{amsthm}
\usepackage{amsmath}
\usepackage{amssymb}
\usepackage{bm}
\usepackage{mathrsfs}
\usepackage[dvips]{graphicx}

\usepackage{algorithm}
\usepackage{algorithmic}

\usepackage{mleftright}
\usepackage{color}

\usepackage{url}

\usepackage{supertabular}

\usepackage{enumitem}

\theoremstyle{definition}

\theoremstyle{remark}

\numberwithin{thm}{section}

\DeclareMathAlphabet{\mathsfsl}{OT1}{cmss}{m}{sl}

\renewcommand{\phi}{\varphi}

\DeclareMathOperator{\tr}{\text{Tr}}

\newcommand{\bZ}{\boldsymbol{Z}}

\newcommand{\bX}{\boldsymbol{X}}

\newcommand{\bP}{\boldsymbol{P}}
\newcommand{\bQ}{\boldsymbol{Q}}

\newcommand{\bU}{\boldsymbol{U}}

\def\b0{\mathbf{0}}

\def\bP{\boldsymbol{P}}
\def\bQ{\boldsymbol{Q}}

\def\bU{\boldsymbol{U}}

\def\bA{\boldsymbol{A}}

\def\bB{\boldsymbol{B}}

\def\bI{\boldsymbol{I}}

\def\nnz{\mathsf{nnz}}

%% file: MatchFame.bbl
\begin{thebibliography}{10}\itemsep=-1pt

\bibitem{viewing_graph}
Federica Arrigoni, Andrea Fusiello, Elisa Ricci, and Tomas Pajdla.
\newblock Viewing graph solvability via cycle consistency.
\newblock In {\em Proceedings of the IEEE/CVF International Conference on
  Computer Vision}, pages 5540--5549, 2021.

\bibitem{inv_semi_group}
Federica Arrigoni, Eleonora Maset, and Andrea Fusiello.
\newblock Synchronization in the symmetric inverse semigroup.
\newblock In {\em International conference on image analysis and processing},
  pages 70--81. Springer, 2017.

\bibitem{Birdal_2019_CVPR}
Tolga Birdal and Umut Simsekli.
\newblock Probabilistic permutation synchronization using the riemannian
  structure of the birkhoff polytope.
\newblock In {\em Proceedings of the IEEE/CVF Conference on Computer Vision and
  Pattern Recognition (CVPR)}, June 2019.

\bibitem{ChatterjeeG13_rotation}
Avishek Chatterjee and Venu~Madhav Govindu.
\newblock Efficient and robust large-scale rotation averaging.
\newblock In {\em {IEEE} International Conference on Computer Vision, {ICCV}
  2013, Sydney, Australia, December 1-8, 2013}, pages 521--528, 2013.

\bibitem{Chen_PPM}
Yuxin Chen and Emmanuel~J. Cand\`es.
\newblock The projected power method: an efficient algorithm for joint
  alignment from pairwise differences.
\newblock {\em Comm. Pure Appl. Math.}, 71(8):1648--1714, 2018.

\bibitem{chen_partial}
Yuxin Chen, Leonidas~J. Guibas, and Qi{-}Xing Huang.
\newblock Near-optimal joint object matching via convex relaxation.
\newblock In {\em Proceedings of the 31th International Conference on Machine
  Learning, {ICML} 2014, Beijing, China, 21-26 June 2014}, pages 100--108,
  2014.

\bibitem{non-seq}
Olof Enqvist, Fredrik Kahl, and Carl Olsson.
\newblock Non-sequential structure from motion.
\newblock In {\em 2011 IEEE International Conference on Computer Vision
  Workshops (ICCV Workshops)}, pages 264--271. IEEE, 2011.

\bibitem{HartleyAT11_rotation}
Richard~I. Hartley, Khurrum Aftab, and Jochen Trumpf.
\newblock {L1} rotation averaging using the weiszfeld algorithm.
\newblock In {\em The 24th {IEEE} Conference on Computer Vision and Pattern
  Recognition, {CVPR} 2011, Colorado Springs, CO, USA, 20-25 June 2011}, pages
  3041--3048, 2011.

\bibitem{Huang13}
Qi{-}Xing Huang and Leonidas~J. Guibas.
\newblock Consistent shape maps via semidefinite programming.
\newblock {\em Comput. Graph. Forum}, 32(5):177--186, 2013.

\bibitem{PPM_vahan}
Vahan Huroyan.
\newblock {\em Mathematical Formulations, Algorithm and Theory for Big Data
  Problems}.
\newblock PhD thesis, University of Minnesota, 2018.

\bibitem{cemp}
Gilad Lerman and Yunpeng Shi.
\newblock Robust group synchronization via cycle-edge message passing.
\newblock {\em arXiv preprint arXiv:1912.11347}, 2019.

\bibitem{sift04}
David~G. Lowe.
\newblock Distinctive image features from scale-invariant keypoints.
\newblock {\em International Journal of Computer Vision}, 60(2):91--110, 2004.

\bibitem{MatchEig}
Eleonora Maset, Federica Arrigoni, and Andrea Fusiello.
\newblock Practical and efficient multi-view matching.
\newblock In {\em 2017 IEEE International Conference on Computer Vision
  (ICCV)}, pages 4578--4586, 2017.

\bibitem{Munkres}
James Munkres.
\newblock Algorithms for the assignment and transportation problems.
\newblock {\em J. Soc. Indust. Appl. Math.}, 5:32--38, 1957.

\bibitem{ozyesil2015robust}
Onur {\"{O}}zyesil and Amit Singer.
\newblock Robust camera location estimation by convex programming.
\newblock In {\em Proceedings of the IEEE Conference on Computer Vision and
  Pattern Recognition}, pages 2674--2683, 2015.

\bibitem{deepti}
Deepti Pachauri, Risi Kondor, and Vikas Singh.
\newblock Solving the multi-way matching problem by permutation
  synchronization.
\newblock In C.~J.~C. Burges, L. Bottou, M. Welling, Z. Ghahramani, and K.~Q.
  Weinberger, editors, {\em Advances in Neural Information Processing Systems
  26}, pages 1860--1868. Curran Associates, Inc., 2013.

\bibitem{SenguptaAGGJSB17}
Soumyadip Sengupta, Tal Amir, Meirav Galun, Tom Goldstein, David~W. Jacobs,
  Amit Singer, and Ronen Basri.
\newblock A new rank constraint on multi-view fundamental matrices, and its
  application to camera location recovery.
\newblock {\em {IEEE} Conference on Computer Vision and Pattern Recognition,
  {CVPR} 2017, Honolulu, Hawaii, USA, June 22-25, 2017}, pages 4798--4806,
  2017.

\bibitem{shen2016}
Tianwei Shen, Siyu Zhu, Tian Fang, Runze Zhang, and Long Quan.
\newblock Graph-based consistent matching for structure-from-motion.
\newblock In {\em European Conference on Computer Vision}, pages 139--155.
  Springer, 2016.

\bibitem{MPLS}
Yunpeng Shi and Gilad Lerman.
\newblock Message passing least squares framework and its application to
  rotation synchronization.
\newblock In {\em Proceedings of the 37th International Conference on Machine
  Learning (ICML)}, 2020.

\bibitem{IRGCL}
Yunpeng Shi, Shaohan Li, and Gilad Lerman.
\newblock Robust multi-object matching via iterative reweighting of the graph
  connection laplacian.
\newblock {\em Advances in Neural Information Processing Systems},
  2020-December, 2020.

\bibitem{FCC}
Yunpeng Shi, Shaohan Li, Tyler Maunu, and Gilad Lerman.
\newblock Scalable cluster-consistency statistics for robust multi-object
  matching.
\newblock In {\em 2021 International Conference on 3D Vision (3DV)}, pages
  352--360. IEEE, 2021.

\bibitem{singer2011angular}
Amit Singer.
\newblock Angular synchronization by eigenvectors and semidefinite programming.
\newblock {\em Applied and computational harmonic analysis}, 30(1):20--36,
  2011.

\bibitem{photo_tourism}
Noah Snavely, Steven~M Seitz, and Richard Szeliski.
\newblock Photo tourism: exploring photo collections in 3d.
\newblock In {\em ACM Siggraph 2006 Papers}, pages 835--846. 2006.

\bibitem{wang_singer_2013}
Lanhui {Wang} and Amit {Singer}.
\newblock Exact and stable recovery of rotations for robust synchronization.
\newblock {\em Information and Inference}, 2013.

\bibitem{ConsistentFeature}
Qianqian Wang, Xiaowei Zhou, and Kostas Daniilidis.
\newblock Multi-image semantic matching by mining consistent features.
\newblock In {\em {IEEE} Conference on Computer Vision and Pattern Recognition,
  {CVPR} 2018, Salt Lake City, UT, USA, June 18-22, 2018}, 2018.

\bibitem{1dsfm14}
Kyle Wilson and Noah Snavely.
\newblock Robust global translations with 1dsfm.
\newblock In {\em Computer Vision - {ECCV} 2014 - 13th European Conference,
  Zurich, Switzerland, September 6-12, 2014, Proceedings, Part {III}}, pages
  61--75, 2014.

\bibitem{Zach2010}
Christopher Zach, Manfred Klopschitz, and Marc Pollefeys.
\newblock Disambiguating visual relations using loop constraints.
\newblock In {\em The Twenty-Third {IEEE} Conference on Computer Vision and
  Pattern Recognition, {CVPR} 2010, San Francisco, CA, USA, 13-18 June 2010},
  pages 1426--1433, 2010.

\bibitem{MatchALS}
Xiaowei Zhou, Menglong Zhu, and Kostas Daniilidis.
\newblock Multi-image matching via fast alternating minimization.
\newblock In {\em {IEEE} International Conference on Computer Vision, {ICCV}
  2015}, 2015.

\end{thebibliography}
